\documentclass[10pt, conference, letterpaper]{IEEEtran}
\IEEEoverridecommandlockouts

\usepackage{textcomp}
\usepackage{xcolor}
\usepackage{listings}
\usepackage{indentfirst}
\usepackage{amsfonts}
\usepackage{algorithm}
\usepackage{algorithmic}
\usepackage{graphicx}
\usepackage{graphbox}
\usepackage{amsmath,bm}
\usepackage{amsthm}
\usepackage{bbm}
\usepackage{dsfont}
\usepackage{stfloats}

\usepackage[utf8]{inputenc}
\usepackage[english]{babel}

\newtheorem{theorem}{Theorem}

\newtheorem{lemma}{Lemma}
\newtheorem{definition}{Definition}

\newtheorem{proposition}{Proposition}
\newtheorem{assumption}{Assumption}

\usepackage{hyperref}
\usepackage{cleveref}
\usepackage{caption}
\usepackage{subfigure}
\usepackage[section]{placeins}
\usepackage{float}
\usepackage{multirow}
\usepackage{cite}
\usepackage{cases}
\usepackage{balance}

\hypersetup{
colorlinks=true,
linkcolor=black,
citecolor=black,
urlcolor=black
}

\def\BibTeX{{\rm B\kern-.05em{\sc i\kern-.025em b}\kern-.08em
    T\kern-.1667em\lower.7ex\hbox{E}\kern-.125emX}}
\begin{document}
\captionsetup[figure]{labelfont={default},labelformat={default},labelsep=period,name={Fig.}}

\title{Theory of Mixture-of-Experts for Mobile Edge Computing}

\author{\IEEEauthorblockN{Hongbo Li}
\IEEEauthorblockA{\textit{Engineering Systems and Design Pillar} \\
\textit{Singapore University of Technology and Design}\\
Singapore \\
hongbo\_li@mymail.sutd.edu.sg}
\and
\IEEEauthorblockN{Lingjie Duan}
\IEEEauthorblockA{\textit{Engineering Systems and Design Pillar} \\
\textit{Singapore University of Technology and Design}\\
Singapore \\
lingjie\_duan@sutd.edu.sg}
\thanks{This work has been accepted by INFOCOM 2025.}
}
\maketitle

\begin{abstract}
In mobile edge computing (MEC) networks, mobile users generate diverse machine learning tasks dynamically over time. These tasks are typically offloaded to the nearest available edge server, by considering communication and computational efficiency. 
However, its operation does not ensure that each server specializes in a specific type of tasks and leads to severe overfitting or catastrophic forgetting of previous tasks. 
To improve the continual learning (CL) performance of online tasks, we are the first to introduce mixture-of-experts (MoE) theory in MEC networks and save MEC operation from the increasing generalization error over time. 
Our MoE theory treats each MEC server as an expert and dynamically adapts to changes in server availability by considering data transfer and computation time. 
Unlike existing MoE models designed for offline tasks, ours is tailored for handling continuous streams of tasks in the MEC environment. 
We introduce an adaptive gating network in MEC to adaptively identify and route newly arrived tasks of unknown data distributions to available experts, enabling each expert to specialize in a specific type of tasks upon convergence.
We derived the minimum number of experts required to match each task with a specialized, available expert.
Our MoE approach consistently reduces the overall generalization error over time, unlike the traditional MEC approach.
Interestingly, when the number of experts is sufficient to ensure convergence, adding more experts delays the convergence time and worsens the generalization error.
Finally, we perform extensive experiments on real datasets in deep neural networks (DNNs) to verify our theoretical results.
\end{abstract}

\section{Introduction}
In mobile edge computing (MEC) networks, mobile users randomly arrive over time to offload intensive machine learning tasks with unknown data distributions to edge servers with superior computing capabilities (e.g., \cite{guo2022distributed,zheng2023federated}). 
Existing offloading and computing strategies in MEC literature typically involve transferring data tasks to the nearest or most powerful available servers by considering communication and computation efficiency within the network (e.g., \cite{ouyang2018follow,shakarami2020survey,gao2019winning,yan2021pricing}). 
However, these MEC approaches do not account for online tasks' data distribution across servers, leading to severe overfitting and degraded performance on previous or other tasks.
This issue, known as catastrophic forgetting in continual learning, can severely degrade learning performance (e.g., \cite{mccloskey1989catastrophic,kirkpatrick2017overcoming,lin2022beyond}). 

Generalization error is a typical measure of how well a model performs on unseen data while retaining knowledge of previous ones (\!\!\cite{nadeau1999inference,ueda1996generalization}). 
Our theoretical analysis later validates that the existing MEC offloading strategies result in a large overall generalization error that increases over time. 
Therefore, it is necessary to design a new task offloading/routing strategy that allows different servers to specialize in specific types of tasks, thereby mitigating the generalization error in MEC networks.

\begin{figure}[t]
    \centering
    \captionsetup{font={footnotesize}}
    \includegraphics[width=0.4\textwidth]{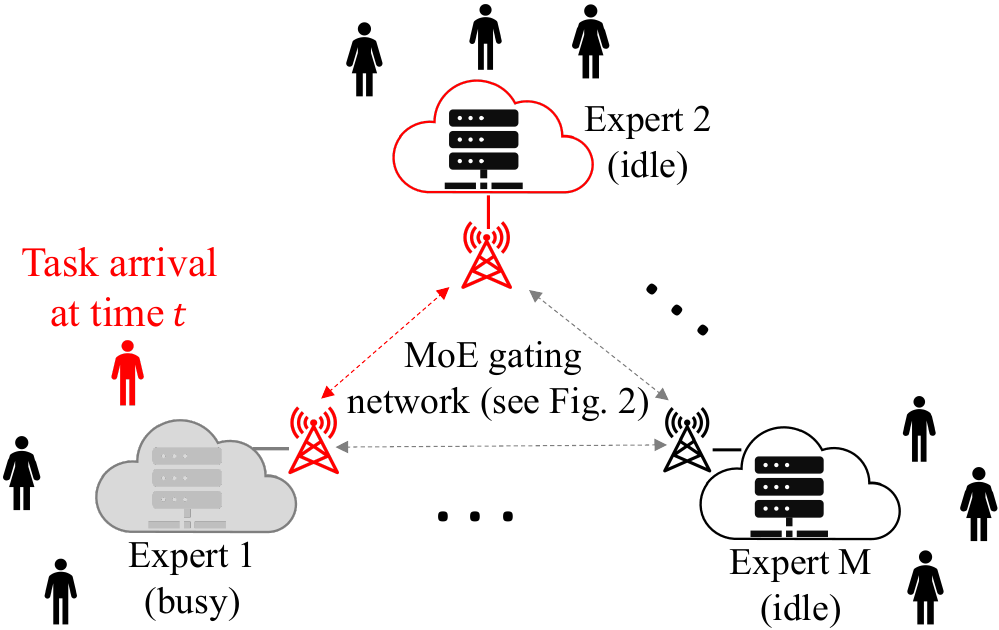}
    \footnotesize
    \caption{An illustration of MEC networks with $M$ edge servers as experts. At the beginning of time $t$, a mobile user arrives to request a task-training service from its nearest Base Station (BS) of expert $\tilde{m}_t$ (e.g., $\tilde{m}_t=1$ in this case). Then our adaptive MoE gating network in Fig.~\ref{fig:gating} selects one idle expert $m_t$ (e.g., $m_t=2$ in this case) out of $M$ experts and asks BS of expert $\tilde{m}_t$ to forward the task dataset to BS of the chosen expert $m_t$. After completing the task learning, the selected expert $m_t$ updates its local model and transmits the training result back to the mobile user via expert $\Tilde{m}_t$'s BS. Finally, the MoE updates the gating network for subsequent task use (see Fig.~\ref{fig:gating}).}
    \label{fig:MEC}
\end{figure}

To reduce the generalization error, it is natural and promising to apply the mixture-of-experts (MoE) model in MEC networks, where each edge server acts as an expert trained with specific tasks and data, as shown in Fig.~\ref{fig:MEC}.
In recent years, MoE has achieved significant success in deep learning, particularly in outperforming single learning expert in tasks involving Large Language Models (LLMs) (e.g., \cite{eigen2013learning,shazeer2016outrageously,riquelme2021scaling,du2022glam,gale2023megablocks}). 
By identifying and routing learning tasks to different experts through a gating network and a router, each expert in MoE specializes in specific knowledge within the data (\!\!\cite{fedus2022switch,chen2022towards,li2024theory}). 
For example, \cite{fedus2022switch} proposes the sparsified MoE model, where the router selects one expert at a time to handle the data of a task, achieving low computational costs while maintaining high model quality. 
Further, \cite{chen2022towards} theoretically analyzes the mechanism of MoE in deep learning within the context of a mixture of classification problems. 
These MoE studies are largely conducted offline, focusing on learning from a pre-collected dataset. 
In MEC networks, however, tasks arrive spontaneously, and each expert must return training results immediately. 
This online learning requirement poses additional challenges for the design of MoE theory and model in MEC.


There are only a few works leveraging MoE in MEC networks by training each expert to handle a particular set of tasks (e.g., \cite{rajbhandari2022deepspeed,singh2023hybrid,wang2024toward,yu2024moesys}). 
For instance, \cite{rajbhandari2022deepspeed} focuses on minimizing network delay by leveraging high-bandwidth connections to allocate experts efficiently. 
\cite{singh2023hybrid} optimizes communication to enhance the routing strategy of the gating network and eliminate unnecessary data movement.
\cite{wang2024toward} exploits the MEC structure to support MoE-based generative AI by optimally scheduling tasks to experts with varying computational resource limitations. 
However, all these studies center on experimental investigations and validation of their designs, leaving a big gap in the theoretical understanding of MoE and its design to guarantee the convergence of continual learning and generalization errors.

In this paper, we aim to bridge this gap by offering new theory and designs of MoE in MEC. To achieve this, we must tackle the following two key challenges.
\begin{itemize} 
    \item The first challenge is \emph{the MoE gating network design to accommodate MEC server availability and continual learning:} 
    In MEC networks, due to data transmission overhead and computation delay, once the MoE gating network selects an expert to handle task data, this expert becomes busy and unavailable until completing the current task.
    Consequently, the design of the MoE gating network and its routing algorithm need to adaptively consider the availability of all MEC servers.
    This requirement significantly departs from existing MoE works that assume all experts are available and can complete their assigned tasks promptly in every round (e.g., \cite{shazeer2016outrageously,fedus2022switch,du2022glam}). Further, the exisiting MoE studies focus on offline learning from a pre-collected dataset, while MEC tasks are dynamically generated to address online. 
    \item The second challenge is \emph{analysis difficulty on learning convergence and generalization error:} 
    Under the MEC features of server availability and continual learning, the updates of expert models and gating network parameters become asynchronous to ensure system convergence. 
    This makes the learning convergence analysis in MEC networks more complex than in the existing MoE models (e.g., \cite{chen2022towards,wang2022coscl}).
    Additionally, it is crucial to theoretically derive and guarantee the overall generalization error to a small constant through new MoE gating network design, which is missing in the current literature on MoE in MEC (e.g., \cite{rajbhandari2022deepspeed,singh2023hybrid,wang2024toward,yu2024moesys}). 
\end{itemize}

Our paper aims to overcome these challenges, and our main contributions and key novelty are summarized as follows.

\begin{itemize}
    \item \emph{New theory of mixture-of-experts (MoE) in mobile edge computing (MEC) networks:} 
    To the best of our knowledge, this paper is the first to introduce mixture-of-experts (MoE) theory in MEC networks and save MEC operation from the increasing generalization error over time. 
    We practically consider that data distributions of arriving tasks are unknown to the system (\Cref{section2}).
    Our MoE theory treats each MEC server as an expert and dynamically adapts to changes in server availability by considering data transfer and computation time. 
    Unlike existing MoE alogrithms designed for offline tasks (\!\!\cite{shazeer2016outrageously,fedus2022switch,chen2022towards}), ours is tailored for handling continuous streams of tasks in the MEC environment. 
    Furthermore, we provide new theoretical performance guarantee of MoE in MEC networks, which are absent in the existing empirical literature on MoE in MEC (\!\! \cite{rajbhandari2022deepspeed,singh2023hybrid,wang2024toward,yu2024moesys}). 
    \item \emph{Adaptive gating network (AGN) design in MoE for expert specialization in MEC:} 
    To address transmission and computation delays, in \Cref{section3}, we introduce an adaptive gating network in MEC to adaptively identify and route newly arrived tasks of unknown data distributions to available experts.
    After training the task at the chosen expert, we use the resulting model error as input to proactively update the gating network parameters for subsequent task arrivals, enabling each expert to specialize in a specific type of tasks after convergence. In \Cref{section4}, our analysis of task routing strategy derives a lower bound on the number of experts to ensure that at least one idle expert specializes in the current task arrival, thereby guaranteeing system convergence after sufficient training rounds for router exploration and expert learning.
    \item \emph{Proved Convergence to constant overall generalization error:} 
    In \Cref{section5}, we first derive that existing MEC's task offloading/routing solutions result in a large generalization error that increases over time. 
    Note that even for current MEC offloading solutions, there is no theoretical analysis on their generalization errors (e.g., \cite{guo2022distributed,zheng2023federated}). 
    After the learning convergence of our MoE model, we rigorously prove that our MoE solution's overall generalization error is upper bounded by a small constant that decreases over time.
    Interestingly, when the number of experts is sufficient to ensure convergence, adding more experts delays the convergence time and worsens the generalization error.
    The benefit of our MoE solution becomes more pronounced when the data distributions vary significantly across all task arrivals.
    Finally, in \Cref{section6}, we perform extensive experiments on real datasets in DNNs to verify our theoretical results.
\end{itemize}


\section{System Model and Problem Setting}\label{section2}

As shown in Fig.~\ref{fig:MEC}, a mobile edge computing (MEC) network operator manages a set $\mathbb{M}=\{1,\cdots,M\}$ of MEC servers or cloudlets as MoE experts, each colocated with a local Base Station (BS) to serve mobile users' continual ML tasks. 
These BSs are typically connected via high-speed fiber optic cables (\!\!\cite{anzola2021joint,ouyang2018follow}).
We consider a discrete time horizon $\mathbb{T}=\{1,\cdots, T\}$. 
In the following, we first introduce the MEC model with continual/online task arrivals, which will be shown to perform bad under the existing task offloading strategies. To improve the learning performance of the system, we then introduce our MoE model with adaptive routing strategy. 

\subsection{MEC Model under Continual Learning}
At the beginning of each time $t$, a mobile user randomly arrives to request task offloading to solve a machine learning problem from its nearest BS. 
Let $\Tilde{m}_t$ denote the nearest expert for the current task $t$ (e.g., $\Tilde{m}_t=1$ in Fig.~\ref{fig:MEC}). 
This user needs to upload its data, denoted by $\mathcal{D}_t=(\mathbf{X}_t,\mathbf{y}_t)$, to the BS of expert $\Tilde{m}_t$ for later task training. Here $\mathbf{X}_t\in\mathbb{R}^{p\times s}$ is the feature matrix with $s$ samples of $p$-dimensional vectors, and $\mathbf{y}_t\in\mathbb{R}^{s}$ is the output vector. 
Note that the data distribution of $\mathbf{X}_t$ is unknown to the MEC network operator. 
For different types of tasks, distributions of their feature matrix $\mathbf{X}_t$'s can vary significantly. While for tasks of the same type, their $\mathbf{X}_t$'s can be very similar. 
After uploading data $\mathcal{D}_t$, the MEC network operator selects one available expert out of $M$ experts, denoted by $m_t\in\mathbb{M}$ (e.g., $m_t=2$ in Fig.~\ref{fig:MEC}), and asks BS of expert $\Tilde{m}_t$ to forward the task dataset to the chosen expert $m_t$. 
Once completing the training of task $t$, expert $m_t$ updates its local model and outputs the learning result to the MEC network operator, and its BS then transmits the result back to the mobile user via expert $\tilde{m}_t$'s BS. 

In the above continual learning (CL) process, after the MEC network operator decides expert $m_t$ for handling task $t$, expert $\tilde{m}_t$ will forward the dataset of task $t$ to expert $m_t$ via its BS' fibre connection to expert $m_t$'s BS.
Thus, there will be a transmission delay for task $t$, which is denoted by $d^{tr}_{t}(m_t,\Tilde{m}_t)\in\{d^{tr}_l,\cdots, d^{tr}_u\}$. Here $d^{tr}_l\geq 0$ and $d^{tr}_u>0$ are the lower and upper bounds of $d^{tr}_{t}(m_t,\Tilde{m}_t)$ satisfying $d^{tr}_l<d^{tr}_u$. 
In MEC networks, this delay $d^{tr}_{t}(m_t,\Tilde{m}_t)$ includes two parts: the uplink data transmission time from the user to the BS $\tilde{m}_t$ and the communication time from BS of expert $\tilde{m}_t$ to BS of expert $m_t$.
Therefore, $d^{tr}_{t}(m_t,\Tilde{m}_t)$ is stochastic and related with the locations of experts $\Tilde{m}_t$ and $m_t$ and the uplink channel condition in BS $\tilde{m}_t$ (\!\!\cite{ouyang2018follow,yan2021pricing}). We practically model that it satisfies a general cumulative distribution function (CDF) distribution and can be different from the others.

For each expert in MEC networks, its computational resource or speed is limited. Consequently, in addition to the transmission delay $d^{tr}_{t}(m_t,\Tilde{m}_t)$, task $t$ takes expert $m_t$ execution time, denoted by $d^{ex}(m_t)\in\{d^{ex}_l,\cdots, d^{ex}_u\}$, to complete the training process. 
Here $d^{ex}_l\geq 0$ and $d^{ex}_u>0$ are the lower and upper bounds of $d^{ex}(m_t)$. 
Similarly, we practically model that $d^{ex}(m_t)$ of any expert $m_t$ satisfies a general CDF distribution. Note that the result return time from expert $m_t$ to user $t$ has no effect on the dynamics of our MoE gating network or routing system. 

In summary, the total time delay for transmitting and training task $t$ is
\begin{align}\label{time_delay}
    d_{t}(m_t,\Tilde{m}_t)=d^{tr}_{t}(m_t,\Tilde{m}_t)+d^{ex}(m_t), 
\end{align}
where $d_{t}(m_t,\Tilde{m}_t)\in\{d^{tr}_l+d^{ex}_l,\cdots, d^{tr}_u+d^{ex}_u\}$. To simplify the notations, we let $d_t=d_{t}(m_t,\Tilde{m}_t)$ and $d_u=d^{tr}_u+d^{ex}_u$ denote the actual total delay for task $t$ and the maximum total delay for any task $t$, respectively.

After being selected for transmitting dataset and training task $t$, expert $m_t$ will remain busy until completing the training process at time $t+d_t$\footnote{We can easily extend our MoE theory to consider that an expert can handle a number of tasks simultaneneously, by checking the residual computation capacity for routing task at each time $t$.}. 
In CL, an online task should be processed immediately and cannot wait a long time to start training (\!\!\cite{shakarami2020survey,gao2019winning}).
Therefore, we define $\gamma_t^{(m)}\in\{0,1\}$ as the binary service state of expert $m\in\mathbb{M}$ at time $t$:
\begin{align}\label{gamma_t}
    \gamma_t^{(m)}=\begin{cases}
        1, &\text{if expert $m$ is idle at time $t$},\\
        0, &\text{if expert $m$ is busy at time $t$.}
    \end{cases}
\end{align}
For example, in Fig.~\ref{fig:MEC}, expert 1 is busy at time $t$ with $\gamma_t^{(1)}=0$, meaning it cannot be selected by the MEC network operator for training the current task $t$. 

Given the availability constraint of $\gamma_t^{(m)}$, it is necessary for the MEC network operator to select appropriate available experts for continually arriving tasks.
However, since the data distribution of each task's feature matrix $\mathbf{X}_t$ (i.e., the task type) is unknown to the MEC network operator, the existing MEC offloading strategy, which always selects the nearest available expert or the most computation-efficient expert (\!\!\cite{ouyang2018follow,shakarami2020survey,gao2019winning,yan2021pricing}), may continually assign tasks with significantly different data distributions to the same expert. 
Our analysis later in \Cref{section5} demonstrates that this approach severely degrades the overall learning performance in CL.
To adaptively identify and route each task of unknown data distribution, we introduce our MoE model to the MEC network in the next subsection.

\subsection{Adaptive Mixture-of-Experts Model in MEC}\label{section2b}

In this subsection, we design the MoE model for the MEC network operator to select appropriate experts for online tasks in CL. 
To identify task types based on feature matrix $\mathbf{X}_t$, as in Fig.~\ref{fig:gating}, we employ a typical linear gating network as a binary classifier and an adaptive router for the MEC network operator (as in \cite{shazeer2016outrageously,fedus2022switch}). 
After the user uploads its task dataset $\mathcal{D}_t=(\mathbf{X}_t,\mathbf{y}_t)$ to its nearest BS $\Tilde{m}_t$ at time $t$, the routing and training of the MoE model illustrated in Fig.~\ref{fig:gating} contain five steps below.

\begin{figure}[t]
    \centering
    \captionsetup{font={footnotesize}}
    \includegraphics[width=0.35\textwidth]{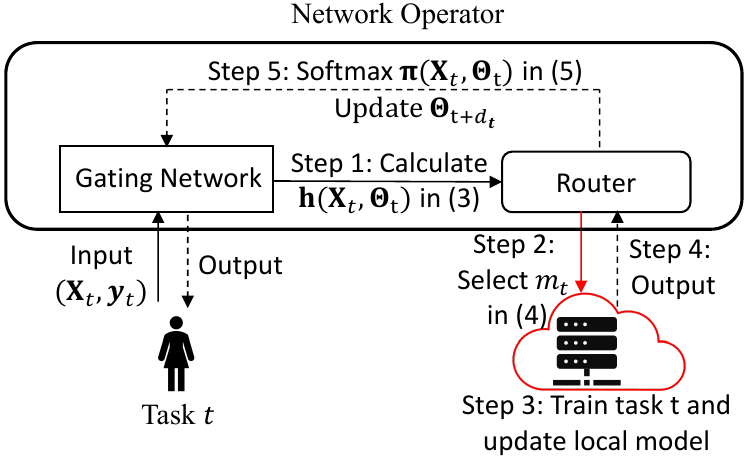}
    \caption{The MoE structure of the MEC network operator in Fig.~\ref{fig:MEC}, which contains a gating network and a router. 
    After a mobile user arrives and uploads its dataset $\mathcal{D}_t$ to the MEC network operator (step~1), the gating network computes its linear output $\mathbf{h}(\mathbf{X}_t,\mathbf{\Theta}_t)$ by (\ref{h_X_theta}) based on the input dataset $\mathcal{D}_t$ (step~2). Then, the router selects the best expert $m_t$ for training task $t$ by the adaptive strategy (\ref{m_t}), based on the gating output $\mathbf{h}(\mathbf{X}_t,\mathbf{\Theta}_t)$ (step~3). After completing the data training, expert $m_t$ updates its local model and    outputs its learning result back to the mobile user (step~4). Finally, the MEC network operator updates gating network parameter $\mathbf{\Theta}_{t+d_t}$ based on the learning result and the softmaxed value $\bm{\pi}(\mathbf{X}_t,\mathbf{\Theta}_t)$ derived in (\ref{softmax}) (step~5).}
    \label{fig:gating}
\end{figure}

\emph{Step 1:} First, the gating network uses $\mathbf{X}_t$ to compute the linear output, denoted by $h_m(\mathbf{X}_t,\bm{\theta}^{(m)}_t)$, for each expert $m\in\mathbb{M}$, where $\bm{\theta}^{(m)}_t\in\mathbb{R}^p$ is the gating network parameter of expert $m$. 
Define $\mathbf{\Theta}_t:=[\bm{\theta}_{t}^{(1)}\ \cdots \ \bm{\theta}_{t}^{(M)}]$
and
$\mathbf{h}(\mathbf{X}_t,\mathbf{\Theta}_t):=[h_1(\mathbf{X}_t,\bm{\theta}_{t}^{(1)})\ \cdots\ h_M(\mathbf{X}_t,\bm{\theta}_{t}^{(M)})]$
to be the parameters and the outputs of the gating network for all experts, respectively. Then we obtain
\begin{align}
    \mathbf{h}(\mathbf{X}_t,\mathbf{\Theta}_t)=\sum_{i\in[s]}\mathbf{\Theta}_t^\top\mathbf{X}_{t,i},\label{h_X_theta}
\end{align}
where $\mathbf{X}_{t,i}$ is the $i$-th sample of the feature matrix $\mathbf{X}_{t}$. 

\emph{Step 2:} Based on the diversified gating output $h_m(\mathbf{X}_t,\bm{\theta}^{(m)}_t)$ of each expert $m$, the router decides which expert to handle the current dataset $\mathcal{D}_t=(\mathbf{X}_t,\mathbf{y}_t)$.
Let $m_t$ denote the expert selected by the router for task $t$. 
To reduce the computational cost and sparsify the gating network, we employ “switch routing" introduced by \cite{fedus2022switch}.
At each time $t$, the router selects expert $m_t$ with the maximum gating network output $h_m(\mathbf{X}_t,\bm{\theta}^{(m)}_t)$, which satisfies
\begin{align}
    m_t=\arg\max_{m\in\mathbb{M},\gamma_m=1} \  \{h_m(\mathbf{X}_t,\bm{\theta}^{(m)}_t)+r_t^{(m)}\}, \label{m_t}
\end{align}
where $r_t^{(m)}=o(1)$ for any expert $m$ is a small random noise to explore different experts in case that some experts have the same gating network output, and $\gamma_m$ is the binary availability state defined in (\ref{gamma_t}). 
In (\ref{m_t}), we can only choose among those available experts with $\gamma_m=1$ in MEC, adding more difficulty in analyzing convergence and learning performance under asynchronous updates compared to existing MoE literature that assumes experts are always available and synchronous   (\!\!\cite{fedus2022switch,chen2022towards,li2024locmoe}).

\emph{Step 3:} After selecting expert $m_t$, the router forwards the dataset $\mathcal{D}_t=(\mathbf{X}_t,\mathbf{y}_t)$ to this expert. Then, this expert trains task $t$ and updates its own local model.

\emph{Step 4:}
Once completing the task training at time $t+d_t$, where $d_t$ is the random total time delay in (\ref{time_delay}), expert $m_t$ returns the learning result to user $t$ via BSs. 

\emph{Step 5:} Finally, after training task $t$, the router calculates the softmaxed gate outputs based on the gating outputs in (\ref{h_X_theta}), derived by
\begin{align}
    \pi_m(\mathbf{X}_t,\mathbf{\Theta}_t)=\frac{\text{exp}(h_m(\mathbf{X}_t,\bm{\theta}^{(m)}_{t}))}{\sum_{m'=1}^M\text{exp}(h_{m'}(\mathbf{X}_t,\bm{\theta}^{(m')}_{t}))},\quad \forall m\in\mathbb{M}.\label{softmax}
\end{align}
Then the MoE model exploits gradient descent to update the gating network parameter $\mathbf{\Theta}_{t+d_{t}}$ for all experts, based on the training loss caused by by expert $m_t$ and the softmaxed values
$\bm{\pi}(\mathbf{X}_t,\mathbf{\Theta}_t)=[\pi_1(\mathbf{X}_t,\mathbf{\Theta}_t) \ \cdots \ \pi_M(\mathbf{X}_t,\mathbf{\Theta}_t)]$. 

The above routing and training process repeats for each continual task $t$. In this process, we aim to diversify $\bm{\theta}_t^{(m)}$ for each expert to specialize in a specific task, enabling the router to assign tasks of the same type to each expert according to (\ref{m_t}).
However, for our MoE model to fit MEC, we cannot adopt the vanilla training process from the existing MoE works, as they typically update $\mathbf{\Theta}_t$ offline using pre-collected datasets (e.g., \cite{shazeer2016outrageously,fedus2022switch,chen2022towards}). 
Our analysis later in \Cref{section3} demonstrates that their update rules do not guarantee learning convergence of $\mathbf{\Theta}_t$ in CL. 
Therefore, we need new designs to train the gating network parameter effectively.

\section{Our Training Process and Algorithm Design for MoE-enabled MEC}\label{section3}
In this section, we describe the training process of our MoE model in \Cref{section2b}. We first present the update rule for the local model of each expert (in Step 3 in Fig.~\ref{fig:gating}), focusing on the standard problems of overparameterized linear regression. Subsequently, we use the updated local expert model as input to proactively update the gating network parameters (Step 5 in Fig.~\ref{fig:gating}) according to our new designs of the loss function and the training algorithm in MEC.

In the following, for a vector $\bm{w}$, we use $\|\bm{w}\|_2$ and $\|\bm{w}\|_{\infty}$ to denote its $\ell$-$2$ and $\ell$-${\infty}$ norms, respectively. As in the MoE literature (e.g., \cite{chen2022towards,li2023theoretical}), for some positive constant $c_1$ and $c_2$, we define $x=\Omega(y)$ if $x>c_2|y|$, $x=\Theta(y)$ if $c_1|y|<x<c_2|y|$, and $x=\mathcal{O}(y)$ if $x<c_1|y|$. We also denote by $x=o(y)$ if $x/y\rightarrow 0$.

\subsection{Update of Local Expert Models for CL}\label{section3a}
Before introducing the update rules of local expert models, we first present the continual task model in our problem.
For each task from a mobile user, we consider fitting a linear model $f(\mathbf{X})=\mathbf{X}^\top \bm{w}$ with ground truth $\bm{w}\in \mathbb{R}^p$ as in the CL literature \cite{evron2022catastrophic,lin2023theory}.
For the task arrival in the $t$-th time slot, it corresponds to a linear regression problem.
For continual learning, overparameterization is more challenging to handle than underparameterization as it can lead to more subtle and complex issues such as overfitting. 
Therefore, in this study, we focus on the overparameterized regime with $s<p$, in which there are numerous linear models that can perfectly fit the data. 
Note that linear regression problems are fundamental in studying ML problems, and we will extend our designs and theoretical insights to DNNs via real-data experiments later in \Cref{section6}.

For the tasks of all mobile users throughout the $T$ time horizon, their ground truths can be classified into $N$ clusters based on task similarity. We require $N<M$ to guarantee the learning convergence of the MoE model, yet we allow the actual task number $T$ to far exceed the expert number $M$. 
Let $\mathcal{W}_n$ denote the $n$-th ground-truth cluster. 
For any two ground truths $\bm{w}_{t}$ and $\bm{w}_{t'}$ in the same cluster $\mathcal{W}_n$, where $t\neq t'$, they should have a higher similarity compared to two ground truths in different clusters. 
Similar to the existing machine learning literature (e.g., \cite{li2023theoretical}), we make the following assumption for any ground truth.
\begin{assumption}\label{assumption_1}
For any two ground different truths, we assume that they satisfy
\begin{align}
    \|\bm{w}_t-\bm{w}_{t'}\|_{\infty}=\begin{cases}
        \mathcal{O}(\sigma_0^{2}),&\text{if }\bm{w}_{t},\bm{w}_{t'}\in \mathcal{W}_n,\\
        \Theta(\sigma_0),&\text{otherwise,}
    \end{cases}\label{w_norm}
\end{align}
where $\sigma_0\in(0,1)$ denotes the variance of ground truths' elements. Moreover, we assume that each ground truth $\bm{w}_t$ possesses a unique feature signal $\bm{v}_t\in\mathbb{R}^d$ with $\|\bm{v}_t\|_{\infty}=\mathcal{O}(1)$.
\end{assumption}

Note that Assumption~\ref{assumption_1}, as specified in (\ref{w_norm}), ensures a significant difference between any two tasks from different clusters. This distinction makes the design of our MoE algorithm more challenging compared to existing MoE literature, which typically assumes smaller task differences (e.g., \cite{fedus2022switch, chen2022towards}).
We will verify the existence of $\sigma_0$ in real datasets later in \Cref{section6}.
Based on Assumption~\ref{assumption_1}, we define the mobile users' generation of dataset for each task $t$.

\begin{definition}\label{def:feature_structure}
At the beginning of each time slot $t\in[T]$, the dataset $\mathcal{D}_t=(\mathbf{X}_{t},\mathbf{y}_{t})$ of the new task arrival is generated by the following steps:

1) Independently generate a random variable $\beta_t\in(0,C]$, where $C$ is a constant satisfying $C=\mathcal{O}(1)$.

2) Generate feature matrix $\mathbf{X}_t$ as a collection of $s$ samples, where one sample is given by $\beta_t \bm{v}_{t}$ and the rest of the $s-1$ samples are drawn from normal distribution $\mathcal{N}(\bm{0},\sigma_t^2\bm{I}_p)$, where $\sigma_t\geq 0$ is the noise level.

3) Generate the output to be $\mathbf{y}_t=\mathbf{X}_t^\top\bm{w}_{t}$.
\end{definition}

According to Assumption~\ref{assumption_1} and \Cref{def:feature_structure}, task $t$ can be classified into one of $N$ clusters based on its feature signal $\bm{v}_t$. 
Since the position of vector $\bm{v}_t$ in feature matrix $\mathbf{X}_t$ is unknown for the MEC network operator, we address this classification sub-problem over $\mathbf{X}_t$ using the adaptive gating network that we introduced in \Cref{section2b}. 

Based on the continual task model, we now introduce the update rules of local expert models.
Let $\bm{w}_t^{(m)}$ denote the local model of expert $m$ at the beginning of $t$-th time slot, where we initialize each model to be zero, i.e., $\bm{w}_0^{(m)}=0$, for any expert $m\in\mathbb{M}$.
Once the router determines expert $m_t$ by the adaptive routing strategy in (\ref{m_t}), it transfers the dataset $\mathcal{D}_t=(\mathbf{X}_t,y_t)$ to this expert for updating $\bm{w}^{(m_t)}_t$. 
Then, expert $m_t$ update its model to $\bm{w}_{t+d_t}^{(m_t)}$ after a random time delay of $d_t$ in (\ref{time_delay}). 
For any other unselected idle expert $m\in \mathbb{M}$ (i.e., $m\neq m_t$), its model $\bm{w}_{t+d_t}^{(m)}$ remains unchanged from latest $\bm{w}_{t+d_t-1}^{(m)}$.

To determine the update rules of $\bm{w}_t^{(m)}$, we need to define the loss function for the rules to minimize the training loss. 
For each task $t$, we define the training loss as the mean-squared error (MSE) with respect to dataset $\mathcal{D}_t$:
\begin{align}
    \mathcal{L}^{tr}_t(\bm{w}_{t+d_t}^{(m_t)},\mathcal{D}_t)=\frac{1}{s}\|(\mathbf{X}_t)^{\top}\bm{w}_{t+d_t}^{(m_t)}-\mathbf{y}_t\|_2^2.\label{training_loss}
\end{align}
Since we focus on the overparameterized regime, there are infinitely many solutions that perfectly make $\mathcal{L}^{tr}_t(\bm{w}_{t+d_t}^{(m_t)},\mathcal{D}_t)=0$ in (\ref{training_loss}).
Among these solutions of $\bm{w}_{t+d_t}^{(m_t)}$, gradient descent (GD) provides a unique solution, which starts from the previous expert model $\bm{w}_{t-1}^{(m_t)}$ at the convergent point, for minimizing $\mathcal{L}^{tr}_t(\bm{w}_{t+d_t}^{(m_t)},\mathcal{D}_t)$ in (\ref{training_loss}). 
According to \cite{gunasekar2018characterizing,evron2022catastrophic,lin2023theory}, this solution is determined by the following optimization problem:
\begin{align}
    \min_{\bm{w}}&\ \ \|\bm{w}-\bm{w}_{t-1}^{(m_t)}\|_2, \label{optimization_SGD} \\
    \text{s.t.}&\ \ \mathbf{X}_t^\top \bm{w}=\mathbf{y}_t.\notag 
\end{align}

Solving (\ref{optimization_SGD}), we derive the update rule of each expert $m$ in the MoE model for task $t$ in the following lemma.
\begin{lemma}\label{lemma:update_w}
For the selected expert $m_t$, after completing task~$t$ at time $t+d_t$, its expert model is updated to be (\!\!\cite{gunasekar2018characterizing,evron2022catastrophic,lin2023theory}):
\begin{align}
    \bm{w}_{t+d_t}^{(m_t)}=\bm{w}_{t+d_t-1}^{(m_t)}+\mathbf{X}_t(\mathbf{X}_t^\top \mathbf{X}_t)^{-1}(\mathbf{y}_t-\mathbf{X}_t^{\top}\bm{w}_{t+d_t-1}^{(m_t)}).\label{update_wt}
\end{align}
While for any other expert $m\neq m_t$, we keep its model unchanged at time $t+d_t$, i.e., 
\begin{align}
    \bm{w}_{t+d_t}^{(m)}=\bm{w}_{t+d_t-1}^{(m)},\quad \forall m\in\mathbb{M} \text{ and } m\neq m_t.\label{noupdate_wt}
\end{align}
\end{lemma}
The proof of \Cref{lemma:update_w} is given in \cite{gunasekar2018characterizing,evron2022catastrophic,lin2023theory} so we skip it here.
Note that from time $t$ to $t+d_t-1$, expert $m_t$ is busy and has no update to its model, such that we have
\begin{align*}
    \bm{w}^{(m_t)}_{t+d_t-1}=\bm{w}^{(m_t)}_{t+d_t-2}=\cdots=\bm{w}^{(m_t)}_{t}.
\end{align*}
However, the other experts may have updates after completing their respective tasks during this period.

\subsection{Training Algorithm Design of Gating Parameters}
After completing the update of local expert model $\bm{w}_{t+d_t}^{(m_t)}$ of expert $m_t$ in Step 3 of Fig.~\ref{fig:gating}, we can use it as input to proactively update the gating network parameter $\mathbf{\Theta}_{t+d_t+1}$ for all experts. 
Our objective for $\bm{\theta}_{t+d_t+1}^{(m)}$ is to enable each expert $m$ specialize in a specific cluster of tasks, aiding the router in selecting the correct expert for each task and reducing learning loss. 
To accomplish this, other than the training loss in (\ref{training_loss}), we follow the existing MoE literature (e.g, \cite{fedus2022switch,li2024locmoe}) to propose another locality loss based on the model error of updating expert model $\bm{w}_{t+d_t}^{(m)}$ in (\ref{update_wt}), which is necessary for training and diversifying $\mathbf{\Theta}_{t+d_t+1}$.
\begin{definition}\label{def:locality_loss}
Given gating network parameter $\mathbf{\Theta}_t$ at time $t$, the locality loss of experts caused by training task $t$ is:
\begin{align}
    \mathcal{L}_t^{loc}(\bm{w}_{t+d_t}^{(m_t)},\mathbf{\Theta}_t)=\sum_{m\in\mathbb{M}}\pi_{m}(\mathbf{X}_t,\bm{\Theta}_t)\|\bm{w}_{t+d_t}^{(m)}-\bm{w}_{t+d_t-1}^{(m)}\|_2,\label{locality_loss}
\end{align}
where $\pi_{m}(\mathbf{X}_t,\bm{\Theta}_t)$ is the softmax value of expert $m$ derived at time $t$ by (\ref{softmax}). 
\end{definition}
By observing (\ref{update_wt}), the locality loss $\mathcal{L}_t^{loc}(\bm{w}_{t+d_t}^{(m_t)},\mathbf{\Theta}_t)$ in (\ref{locality_loss}) is minimized when the MoE routes tasks within the same ground-truth cluster $\mathcal{W}_n$ in (\ref{w_norm}) to the same experts. Therefore, $\mathcal{L}_t^{loc}(\bm{w}_{t+d_t}^{(m_t)},\mathbf{\Theta}_t)$ helps accelerate diversifying gating network parameters for all experts. Without (\ref{locality_loss}), $\bm{\theta}^{(m)}_t$ of each expert $m\in\mathbb{M}$ cannot be efficiently diversified to specialize in different types of tasks.

Based on (\ref{training_loss}) and (\ref{locality_loss}), we finally define the task loss objective for each task $t$ as follows:
\begin{align}\label{task_loss}
    \mathcal{L}_t(\bm{w}_{t+d_t}^{(m_t)},\mathbf{\Theta}_t,\mathcal{D}_t)=\mathcal{L}^{tr}_t(\bm{w}_{t+d_t}^{(m_t)},\mathcal{D}_t)+\mathcal{L}_t^{loc}(\bm{w}_{t+d_t}^{(m_t)},\mathbf{\Theta}_t).
\end{align}
Commencing from the initialization $\mathbf{\Theta}_0$, the gating network parameter is updated based on gradient descent:
\begin{align}
    \bm{\theta}_{t+d_t+1}^{(m)}=\bm{\theta}_{t+d_t}^{(m)}-\eta \cdot \nabla_{\bm{\theta}_t^{(m)}}  \mathcal{L}_t(\bm{w}_{t+d_t}^{(m_t)},\mathbf{\Theta}_t,\mathcal{D}_t),\ \forall m\in\mathbb{M}, \label{update_theta}
\end{align}
where $\eta>0$ is the learning rate. Note that $\bm{w}_{t+d_t}^{(m_t)}$ in (\ref{update_wt}) is also the optimal solution for minimizing the task loss in (\ref{task_loss}), as the locality loss is derived after updating $\bm{w}_{t+d_t}^{(m_t)}$ in (\ref{update_wt}).
This is adopted by the existing MoE literature (e.g., \cite{fedus2022switch,chen2022towards}) for offline task training.

\begin{algorithm}[t]
\caption{Adaptive routing and training for MoE gating network}
\small
\label{algo:update_MoE}
\begin{algorithmic}[1]
\STATE \textbf{Input:} $T,\sigma_0, \delta=o(1)$;\
\STATE Initialize $\bm{\theta}^{(m)}_0=\bm{0}$ and $\bm{w}_0^{(m)}=\bm{0}$, $\forall m\in\mathbb{M}$;\
\FOR{$t=1,\cdots, T$}
\STATE Input dataset $\mathcal{D}_t=(\mathbf{X}_t,\mathbf{y}_t)$;\
\STATE Generate noise $r_t^{(m)}$ for each expert $m\in\mathbb{M}$;\
\STATE Select expert $m_t$ according to (\ref{m_t}) and transmit dataset $\mathcal{D}_t$ to expert $m_t$;\ \label{line_routing}
\STATE Update local expert model $\bm{w}^{(m_t)}_{t+d_t}$ by (\ref{update_wt}) after completing task~$t$;\  \label{line_expert_update}
\STATE Update gating network parameter $\bm{\theta}^{(m)}_{t}$ as in (\ref{update_theta}) for any expert $m\in\mathbb{M}$;\ \label{line_update_theta}
\ENDFOR
\end{algorithmic}
\end{algorithm}

Based on the system settings above, we propose adaptive routing and training \Cref{algo:update_MoE}.
According to \Cref{algo:update_MoE}, at any time $t\in\mathbb{T}$, the MoE model selects the best expert $m_t$ using the adaptive routing strategy (\ref{m_t}) (in Line~\ref{line_routing} of \Cref{algo:update_MoE}) and updates the local model $\bm{w}^{(m_t)}_{t+d_t}$ of expert $m_t$ according to (\ref{update_wt}) (in Line~\ref{line_expert_update}). 
Following that, our MoE updates the gating network parameters to diversify $\bm{\theta}_t^{(m)}$ for any expert $m\in\mathbb{M}$ (in Line~\ref{line_update_theta}). The above process repeats until the arrival of the last task $T$.

\section{Theoretical Guarantee for Learning Convergence of our MoE Design}\label{section4}

In our MoE model in MEC, achieving learning convergence is critical, ensuring that each expert stabilizes to specialize in specific types of tasks.
Upon convergence, the MoE significantly reduces the training loss for future tasks in continual learning.
Therefore, in this section, we derive the lower bound of the expert number required to guarantee the learning convergence of MoE, given the server availability constraint in MEC. 
Based on this result, we prove the learning convergence of \Cref{algo:update_MoE}.

Before analyzing the learning convergence, we first propose the following lemma to mathematically  characterize how the gating network identify task types based on the gating network output of each expert in (\ref{h_X_theta}). 
\begin{lemma}\label{lemma:h-h_hat}
For two feature matrices $\mathbf{X}$ and $\Tilde{\mathbf{X}}$, if their ground truths $\bm{w},\tilde{\bm{w}}\in\mathcal{W}_{n}$ are in the same cluster, with probability at least $1-o(1)$, their corresponding gating network outputs of the same expert $m$ satisfy
\begin{align}\label{|hm-hm_tilde|}
\big|h_m(\mathbf{X},\bm{\theta}_t^{(m)})-h_m(\Tilde{\mathbf{X}},\bm{\theta}_t^{(m)})\big|=\mathcal{O}(\sigma_0).
\end{align}
\end{lemma}

The proof of \Cref{lemma:h-h_hat} is given in \Cref{proof_lemma:h-h_hat}.
\Cref{lemma:h-h_hat} indicates that for any feature matrices with the same feature signal, their gating network outputs exhibit a very small gap of $\mathcal{O}(\sigma_0)$ in (\ref{|hm-hm_tilde|}). This occurs because, according to \Cref{def:feature_structure}, all the other $s-1$ samples of $\mathbf{X}_{t,i}$, except for feature signal $\bm{v}_t$, have a mean value of $\bm{0}$.
Leveraging \Cref{lemma:h-h_hat}, the router can effectively identify the type of each task $t$ based on the gating network output of each expert. 

Motivated by \Cref{lemma:h-h_hat}, given $N$ ground-truth clusters in (\ref{w_norm}), we can classify all experts into $N$ expert sets based on their specialty, where each set $\mathcal{M}_n$ is defined as: 
\begin{align}\label{M_n}
    \mathcal{M}_n=\Big\{m\in\mathbb{M}\Big| &(\bm{\theta}_t^{(m)})^\top \bm{v}_i>(\bm{\theta}_t^{(m)})^\top \bm{v}_j,\\ &\forall \bm{w}_i\in\mathcal{W}_n, \bm{w}_j\notin\mathcal{W}_n\Big\}.\notag
\end{align}

Unlike traditional MoE literature, which assumes all experts are always available (e.g, \cite{fedus2022switch,chen2022towards,li2024locmoe}), each expert server in our context may be occupied or unavailable at any time $t$ due to the random total time delay $d_t$ in (\ref{time_delay}) for transmitting and training. 
In extreme cases, it is possible for a task to miss any previously trained expert of the same type. 
Consequently, we need a lower bound on the number of experts $M$ to ensure the convergence of the MoE model. 
This guarantees that the router can always select an available expert specializing in the new task $t$. 
\begin{proposition}\label{lemma:M_th}
As long as $M=\Omega(NM_{th}\ln(\frac{1}{\delta}))$, where
\begin{align}\label{M_th}
    M_{th}=\frac{d_u}{N}+\Phi^{-1}(1-\delta)\sqrt{\frac{d_u}{N}(1-\frac{1}{N})}
\end{align}
with the standard normal quantile $\Phi^{-1}(\cdot)$ and $\delta=o(1)$, for any task $t$ with $\bm{w}_t\in\mathcal{W}_n$ arrives after the system convergence, with probability at least $1-o(1)$, there always exists an idle expert $m\in \mathcal{M}_{n}$ with $\gamma_t^{(m)}=1$ in (\ref{gamma_t}) to provide the correct type of expertise for that task.
\end{proposition}

The proof of \Cref{lemma:M_th} is given in \Cref{proof_lemma:M_th}.
Next, we demonstrate that our \Cref{algo:update_MoE} works effectively to diversify the gating network parameter for each expert within the adaptive gateway network.

\begin{proposition}[Router's Convergence]\label{prop:exploration}
Under \Cref{algo:update_MoE}, for any task arrival $t>T_1$ with $\bm{w}_t\in \mathcal{W}_n$, where $T_1=d_u+\lceil \eta^{-1}\sigma_0^{-0.5}M\ln(\frac{M}{\delta})\rceil$, with probability at least $1-o(1)$, we obtain
\begin{align}\label{hm-hm'}
    &\|\pi_m(\mathbf{X}_t,\bm{\theta}_t^{(m)})-\pi_{m'}(\mathbf{X}_t,\bm{\theta}_t^{(m')})\|_{\infty}\\=&\begin{cases}
        \mathcal{O}(\sigma_0), &\text{if }m,m'\in\mathcal{M}_n,\\
        \Omega(\sigma_0^{0.5}),&\text{otherwise.}
    \end{cases}\notag
\end{align}
Following (\ref{hm-hm'}), the router consistently assigns tasks within the same ground-truth cluster $\mathcal{W}_n$ to any expert $m\in\mathcal{M}_n$. 
\end{proposition}
The proof of \Cref{prop:exploration} is given in \Cref{proof_prop:exploration}.
By observing (\ref{hm-hm'}), for any two experts $m$ and $m'$ within the same expert set $\mathcal{M}_n$, the gap between their gating network outputs is small, satisfying $\mathcal{O}(\sigma_0)$. 
Conversely, for any two experts in different expert sets, their gating output gap is diversified to be larger, i.e., $\Theta(\sigma_0)$.

According to the update rule of gating network parameters in (\ref{update_theta}), after completing a task at time $t$, the MoE will update $\bm{\theta}_t^{(m)}$ for any expert $m\in\mathbb{M}$ simultaneously.
Therefore, we only need to ensure that the MoE model has completed a sufficient number of tasks (i.e., $\lceil \eta^{-1}\sigma_0^{-0.5}M\ln(\frac{M}{\delta})\rceil$) for router exploration to diversity $\bm{\theta}_t^{(m)}$ of experts within different expert sets.
Given the diversified gating network parameter of each expert in (\ref{hm-hm'}), the router will only select expert $m$ in set $\mathcal{M}_n$ for new tasks within cluster $\mathcal{W}_n$ by our adaptive routing strategy in (\ref{m_t}).

Based on \Cref{prop:exploration}, our \Cref{algo:update_MoE} guarantees that each expert $m$ will continue to correctly update its local model $\bm{w}_t^{(m)}$ with the same type of tasks correctly assigned by the router. 
Then, we analyze the learning convergence of each expert's local model under \Cref{algo:update_MoE} in the next proposition. 

\begin{proposition}[Experts' Learning Convergence]\label{prop:expert_learning}
Under \Cref{algo:update_MoE}, for any task arrival $t>T_1$, each expert $m\in\mathbb{M}$ satisfies learning convergence
\begin{align}
    \|\bm{w}_{t}^{(m)}-\bm{w}_{T_2}^{(m)}\|_{\infty}=\mathcal{O}(\sigma_0^{2}) \label{expert_stabilization}
\end{align}
with probability at least $1-o(1)$.
\end{proposition}

The proof of \Cref{prop:expert_learning} is given in \Cref{proof_prop:expert_learning}.
Based on the update rule of the expert model in (\ref{update_wt}), as long as the router consistently assigns tasks from the same cluster to expert $m$, its model $\bm{w}_t^{(m)}$ will only undergo slight updates for each task, with the minimum gap $\mathcal{O}(\sigma_0^2)$ in (\ref{w_norm}). 
Therefore, after completing a sufficient number of tasks (i.e., $\lceil \eta^{-1}\sigma_0^{-0.5}M\ln(\frac{M}{\delta})\rceil$) for updating each expert $m\in\mathbb{M}$, each expert will complete its learning stage and stabilize within one out of the $N$ expert sets in (\ref{M_n}).

\section{Theoretical Performance Guarantee on Overall Generalization Error}\label{section5}

Based on the convergence analysis in \Cref{section4}, in this section, we derive the overall generalization error of \Cref{algo:update_MoE} to further analyze the learning performance of our MoE model. 
Note that overall generalization error encapsulates the MoE model's ability to learn new tasks while retaining knowledge of previous ones in CL, where a large error tells severe overfitting or catastrophic forgetting of previous tasks (\!\!\cite{chaudhry2018efficient,doan2021theoretical,lin2023theory}).
To demonstrate the performance gain of our MoE solution, we also derive the generalization error of the existing MEC offloading strategy as a benchmark.

For the MoE model described in \Cref{section2}, we follow the existing literature on continual learning (e.g., \cite{chaudhry2018efficient,doan2021theoretical,lin2023theory}) to define $\mathcal{E}_t(\bm{w}_{t+d_t}^{(m_t)})$ as the model error for the $t$-th task:
\begin{align}
  \mathcal{E}_t(\bm{w}_{t+d_t}^{(m_t)})=\|\bm{w}_{t+d_t}^{(m_t)}-\bm{w}_{t}\|^2_2,\label{def:model_error}
\end{align}
which characterizes the generalization performance of the selected expert $m_t$ with model $\bm{w}_{t+d_t}^{(m_t)}$ for task $t$ at round $t$.

As in the continual learning literature (e.g., \cite{chaudhry2018efficient,doan2021theoretical,lin2023theory}), we define the overall generalization performance of the model $\bm{w}_{T+d_T}^{(m)}$ after training the last task $T$ at time $T+d_T$, by computing the average model error in (\ref{def:model_error}) across all tasks:
\begin{align}    
    G_T=\frac{1}{T}\sum_{t=1}^T\mathcal{E}_{d}(\bm{w}_{T+d_T}^{(m_{t})}).\label{def:generalization_error}
\end{align}

In the following proposition, we derive the overall generalization error, defined in (\ref{def:generalization_error}), for the existing offloading strategies in MEC networks, which typically transfer tasks to the nearest or most powerful available experts to enhance communication and computation efficiency (e.g., \cite{ouyang2018follow,shakarami2020survey,gao2019winning,yan2021pricing}). 
Here we define $r:=1-\frac{s}{p}<1$ as the overparameterized ratio, where $s$ is the number of samples and $p$ is the dimension of each sampled vector in $\mathbf{X}_t$.
We define the number of updates of expert $m$ till completing task $t$ as:
\begin{align}
    L_t^{(m)}=\sum_{\tau=1}^t \mathds{1}\{m_{\tau}=m\},\label{L_t(m)}
\end{align}
where $m_{\tau}$ is the selected expert at time $\tau\in\{1,\cdots, t\}$, and $\mathds{1}\{(\cdot)\}=1$ if $(\cdot)$ is true and $\mathds{1}\{(\cdot)\}=0$ otherwise.
For expert $m$, let $\tau^{(m)}(i)\in\{1,\cdots, T\}$ represent the time slot when the router selects expert $m$ for the $i$-th time. 
\begin{proposition}\label{prop:GT_benchmark}
If the MEC network operator always chooses the nearest or the most powerful expert for each task arrival $t\in\mathbb{T}$ as in the existing MEC offloading literature (e.g., \cite{ouyang2018follow,shakarami2020survey,gao2019winning,yan2021pricing}), the overall generalization error is:
\begin{align}
    \mathbb{E}[G_T]\label{error_benchmark}
    =&\underbrace{\frac{1}{T}\sum_{t=1}^T r^{L_{T}^{(m_t)}}\|\bm{w}_t\|^2}_{\text{term G}^1}\\&+\underbrace{\frac{1}{T}\sum_{t=1}^T(1-r^{L_T^{(m_t)}})\mathbb{E}\Big[\|\bm{w}_{n}-\bm{w}_{n'}\|^2\Big|n,n'\in[N]\Big]}_{\text{term G}^2},\notag
\end{align}
which approaches to the maximum $\mathbb{E}[\|\bm{w}_{n}-\bm{w}_{n'}\|^2|n,n'\in[N]]$ as time horizon $T\rightarrow\infty$.
\end{proposition}
The proof of \Cref{prop:GT_benchmark} is given in \Cref{proof_prop:GT_benchmart}.
The generalization error arises from two terms in (\ref{error_benchmark}):
\begin{itemize}
    \item Term $\text{G}^1$: training error of the ground truth of each task under overparameterized regime, which does not scale up with time horizon $T$ given overparameterized ratio $r<1$.
    \item Term $\text{G}^2$: the model gap between ground truths assigned the same expert, which increases with time horizon $T$. As $T\rightarrow\infty$, $\text{G}^2$ approaches to the maximum $\mathbb{E}[\|\bm{w}_{n}-\bm{w}_{n'}\|^2|n,n'\in[N]]$.
\end{itemize}
Under the existing offloading strategies (e.g., \cite{ouyang2018follow,shakarami2020survey,gao2019winning,yan2021pricing}), the MEC network operator prioritizes communication and computation efficiency, resulting in the random assignment of task types to each expert regardless of its assigned task types in the past. 
Consequently, none of the expert models can fully converge, and they continue to update until the final task $T$, as shown in both terms $\text{G}^1$ and $\text{G}^2$ in (\ref{error_benchmark}). 
Once distinct tasks exhibit large model gap $\text{G}^2$, the learning performance of $\mathbb{E}[G_T]$ deteriorates. As time horizon $T$ approaches infinity, $\mathbb{E}[G_T]$ finally approaches to the maximum $\mathbb{E}[\|\bm{w}_{n}-\bm{w}_{n'}\|^2|n,n'\in[N]]$ in $\text{G}^2$, which is totally determined by the expected model gap of two distinct tasks.
Thus, \Cref{prop:GT_benchmark} tells that the existing offloading strategies lead to poor learning performance to handle continual task learning with diverse data distributions.

\begin{figure*}[t]
    \centering
    \begin{eqnarray}\label{error_algorithm}
        \mathbb{E}[G_T]
    <& \frac{1}{T}\sum_{t=1}^{T}r^{L_{T}^{(m_t)}}\|\bm{w}_{t}\|^2+\underbrace{\frac{1}{T}\sum_{t=1}^{T}(1-r^{L_{T_1}^{(m_t)}})\cdot r^{L_T^{(m_t)}-L_{T_1}^{(m_t)}}\mathbb{E}\Big[\|\bm{w}_{n}-\bm{w}_{n'}\|^2\Big|n,n'\in[N]\Big]}_{\text{term G}^3}+\underbrace{\mathcal{O}(\sigma_0^{2})}_{\text{term G}^4}
    \end{eqnarray}
    \vspace*{8pt}
    \hrulefill
    \vspace*{8pt}
\end{figure*}

Next, we derive the explicit upper bounds for the overall generalization error of our \Cref{algo:update_MoE} in the following theorem, based on \Cref{lemma:M_th} on number of experts, \Cref{prop:exploration} on our MoE router convergence and \Cref{prop:expert_learning} on our MoE expert convergence.
\begin{theorem}\label{thm:error_algorithm}
Given $M=\Omega(NM_{th}\ln(\frac{1}{\delta}))$ with $M_{th}$ defined in (\ref{M_th}), after \Cref{algo:update_MoE}'s completion of training the last task $T$ at time $T+d_T$, the overall generalization error satisfies (\ref{error_algorithm}), which converges to the minimum model error $\mathcal{O}(\sigma_0^{2})$ between tasks in the same cluster, as time horizon $T\rightarrow\infty$.
\end{theorem}
The proof of \Cref{thm:error_algorithm} is given in \Cref{proof_thm:error_algorithm}.
Compared to generalization error resulting from existing MEC offloading strategies in (\ref{error_benchmark}), $\mathbb{E}[G_T]$ in (\ref{error_algorithm}) successfully decomposes the dominant error term $\text{G}^2$ into two components:
\begin{itemize}
    \item Term $\text{G}^3$: the model error arising from the randomized routing for the router exploration and the expert learning ($t\leq T_1$ in \Cref{prop:exploration} and \Cref{prop:expert_learning}) under \Cref{algo:update_MoE}. As time horizon $T\rightarrow\infty$, $\text{G}^3$ tends toward $0$ given $r<1$, as the long-term convergence mitigates the generalization error caused by wrong routing decisions made in the early stage. 
    \item Term $\text{G}^4$: the minimum model error between similar tasks within the same cluster that are routed to a specific expert $m$ after the expert stabilizes within an expert set, as described in \Cref{prop:expert_learning}.
\end{itemize} 

Therefore, the generalization error of our MoE \Cref{algo:update_MoE} is significantly reduced compared to (\ref{error_benchmark}), especially when $T$ is non-small.
Given a fixed time horizon $T$, increasing the number of experts $M$ decreases the number of updates $L_{T}^{(m)}$ in (\ref{L_t(m)}) for each expert $m\in\mathbb{M}$ up to $T$, while $L_{T_1}^{(m)}$ remains unchanged after the convergence time $T_1$ in \Cref{prop:exploration}. 
Consequently, the right-hand-side of (\ref{error_algorithm}) increases, since $r^{L_{T}^{(m)}}$ decreases with $L_{T}^{(m)}$. 
However, as $T\rightarrow \infty$, $\mathbb{E}[G_T]$ is still bounded by the minimum $\mathcal{O}(\sigma_0^2)$ in (\ref{error_algorithm}).

\begin{figure*}[htbp]
    \centering
    \captionsetup{font={footnotesize}}
    \subfigure[Our \Cref{algo:update_MoE}.]{\label{subfig:error_algo}\includegraphics[width=0.38\textwidth]{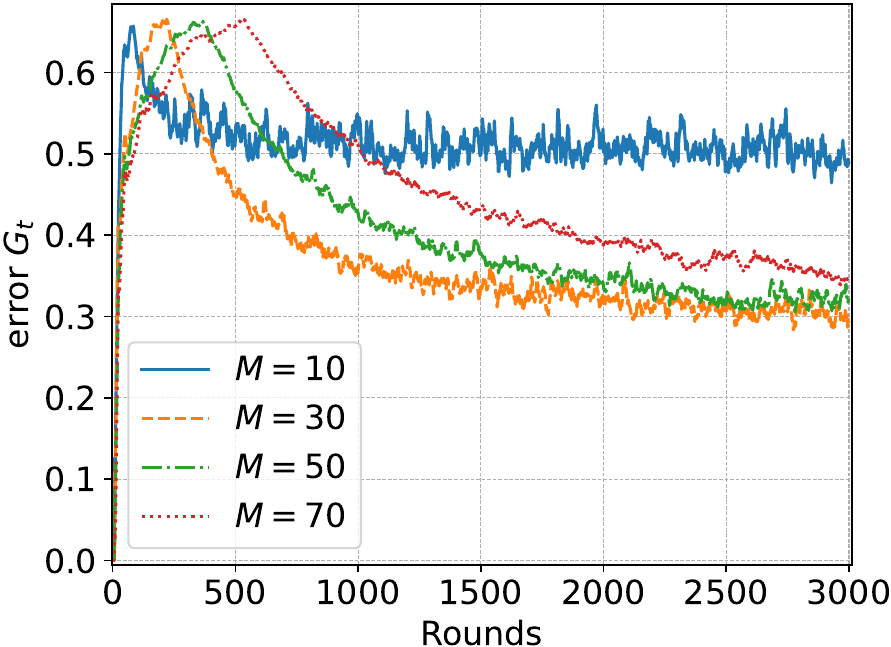}} 
    \hspace{0.2cm}
    \subfigure[Existing MEC offloading strategies.]{\label{subfig:error_benchmark}\includegraphics[width=0.38\textwidth]{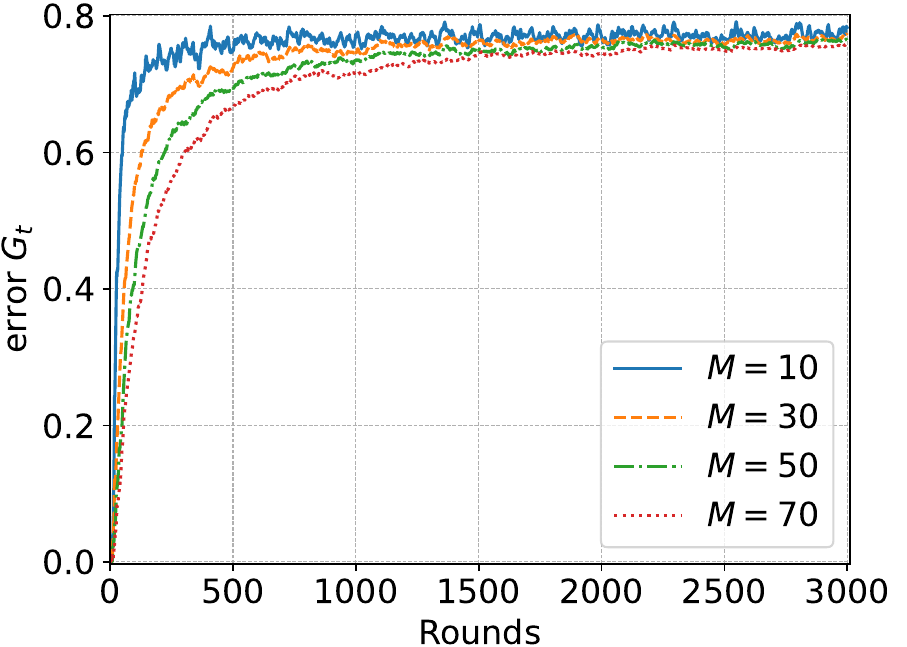}}
    \caption{The dynamics of overall generalization errors under our MoE \Cref{algo:update_MoE} and the MEC existing offloading strategies that always select the nearest or the most powerful available server (e.g., \cite{ouyang2018follow,shakarami2020survey,gao2019winning,yan2021pricing}).
    Here we set $T=3000$, $N=10$, $\sigma_0=0.6,d_u=10, \eta=0.2,p=15,s=10$, and vary $M\in\{10,30,50,70\}$.}
    \label{fig:error_linear}
\end{figure*}

\section{Real data Experiments on DNNs to Verify our MoE Theoretical Results}\label{section6}

In addition to theoretical analysis on learning convergence in \Cref{section4} and performance guarantee in \Cref{section5}, we further conduct extensive experiments to validate our theoretical results. 
Below we first run experiments using linear models in \Cref{section3} to confirm theoretical results. 
Subsequently, we extend to DNNs by employing real datasets, thereby demonstrating the broader applicability of our results.

In the first experiment, we generate synthetic datasets by \Cref{def:feature_structure} in linear models to compare the learning performance of our designed \Cref{algo:update_MoE} against the existing offloading strategies in MEC, which always select the nearest or the most powerful available expert (e.g., \cite{ouyang2018follow,shakarami2020survey,gao2019winning,yan2021pricing}).
We set the parameters as follows: $T=3000$, $N=10$ task types, $\sigma_0=0.6$, $d_u=10$, $\eta=0.2$, $p=15$, and $s=10$. We vary the number of experts $M\in\{10,30,50,70\}$ in Fig.~\ref{fig:error_linear} to compare the overall generalization errors over time of the three approaches.

As shown in Fig.~\ref{subfig:error_algo}, the overall generalization error under our proposed \Cref{algo:update_MoE} decreases with time rounds and becomes smaller than $\sigma_0^2=0.36$ for $M\in\{30,50,70\}$, consistent with (\ref{error_algorithm}) from \Cref{thm:error_algorithm}. 
However, for small server number $M=10$, the MoE approach cannot guarantee the availability of an idle expert with $\gamma_t^{(m)}=1$ to handle a new task correctly, resulting in incorrect routing and increased generalization error. This observation supports \Cref{lemma:M_th} concerning the minimum number of experts.
Additionally, the comparison between $M=50$ and $M=70$ reveals that when the number of experts is sufficient to ensure convergence, adding more experts delays the convergence time and worsens the generalization error within the same time horizon, aligning with our generalization error result in \Cref{thm:error_algorithm}.

In stark contrast, as depicted in Fig.~\ref{subfig:error_benchmark}, the MEC offloading strategy leads to poor performances, with errors obviously exceeding $\sigma_0=0.6$. 
This outcome supports our analysis in \Cref{prop:GT_benchmark}, which respectively indicates that the traditional MEC approaches fail to ensure that each expert specializes in a specific task type, resulting in an increasing expected overall generalization error.

\begin{figure}[t]
    \centering
    \captionsetup{font={footnotesize}}
    \includegraphics[width=0.8\linewidth]{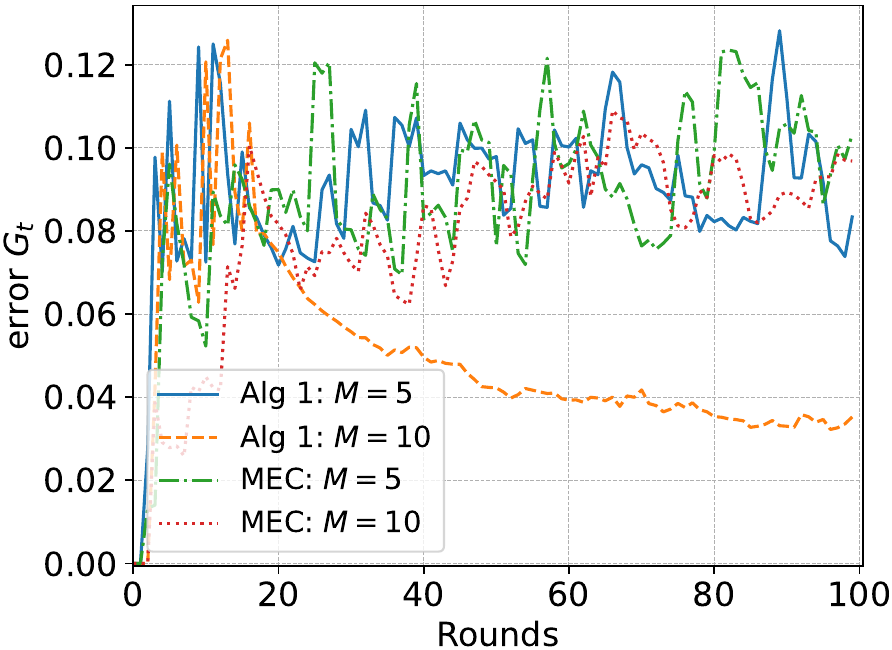}
    \caption{The dynamics of overall generalization errors under our \Cref{algo:update_MoE} and the existing MEC offloading strategies, using DNNs in MNIST datasets \cite{lecun1989handwritten}.}
    \label{fig:NN_MoE_algo}
\end{figure}

Finally, we extend our \Cref{algo:update_MoE} and theoretical results from linear models to DNNs by conducting computation-heavy experiments on the MNIST dataset (\!\!\cite{lecun1989handwritten}).
We use a five-layer neural network, consisting of two convolutional layers and three fully connected layers.
Here we set $T=100$, $N=3$, $d_u=4$, and $\eta=0.2$. We vary $M\in\{5,10\}$ and plot Fig.~\ref{fig:NN_MoE_algo}.
From the dataset, we verify our Assumption~\ref{assumption_1} that the variance among different types of tasks is $\sigma_0=0.1$.
In each round, we construct the feature matrix by averaging $s=100$ training data samples. 
To diversify model gaps of different types of tasks, we transform the $d\times d$ matrix into a $d\times d$ dimensional normalized vector to serve as the feature signal $\bm{v}_t$ in \Cref{def:feature_structure}.
As depicted in Fig.~\ref{fig:NN_MoE_algo}, when the number of expert, $M=5$, is below our derived lower bound in \Cref{lemma:M_th}, both our \Cref{algo:update_MoE} and the MEC offloading strategies fail to select the optimal server with availability constraint, resulting in large errors. However, when $M$ is increased to $10$, our \Cref{algo:update_MoE} significantly outperforms existing MEC strategies, which still select wrong experts.

\section{Conclusions}\label{section7}
To improve the learning performance in MEC with continual task arrivals, this paper is the first to introduce MoE theory in MEC networks and save MEC operation from the increasing generalization error over time. 
Our MoE theory treats each MEC server as an expert and dynamically adapts to changes in server availability by considering data transfer and computation time. 
We introduce an adaptive gating network in MEC to adaptively identify and route newly arrived tasks of unknown data distributions to available experts, enabling each expert to specialize in a specific type of tasks upon convergence.
We derived the minimum number of experts required to match each task with a specialized, available expert.
Our MoE approach consistently reduces the overall generalization error over time, unlike the traditional MEC approach.
Interesting, when the number of experts is sufficient for convergence, adding more experts delays the convergence time and worsens the generalization error within the same time horizon.
Finally, we perform extensive experiments on real datasets in DNNs to verify our theoretical results.

\appendix
\subsection{Proof of Lemma~\ref{lemma:h-h_hat}}\label{proof_lemma:h-h_hat}

For dataset $(\mathbf{X}_t,\mathbf{y}_t)$ generated in Definition~\ref{def:feature_structure} per round $t$, we can assume that the first sample of $\mathbf{X}_t$ is the signal vector. Therefore, we rewrite $\mathbf{X}_t=[\beta_t \bm{v}_{n}\ \mathbf{X}_{t,2}\ \cdots\ \mathbf{X}_{t,s}]$. Let $\tilde{\mathbf{X}}_t=[\beta_t \bm{v}_{n_t}\ 0\ \cdots\ 0]$ represent the matrix that only keeps the feature signal. 

Based on the definition of the gating network in Fig.~\ref{fig:gating},
we have $h_m(\mathbf{X}_t,\bm{\theta}_t^{(m)})=\sum_{i=1}^{s}(\bm{\theta}_t^{(m)})^\top \mathbf{X}_{t,i}$. 
Then we calculate
\begin{align*}
    &\Big|h_m(\mathbf{X}_t,\bm{\theta}_t^{(m)})-h_m(\tilde{\mathbf{X}}_t,\bm{\theta}_t^{(m)})\Big|\\=&\Big|(\bm{\theta}_t^{(m)})^\top\sum_{i=1}^{s}\mathbf{X}_{t,i}\Big|\\
    \leq& \Big|(\bm{\theta}_t^{(m)})^\top\sum_{i=2}^{s}\mathbf{X}_{t,i}\Big|+\Big|(\bm{\theta}_t^{(m)})^\top(\bm{v}_n-\bm{v}_{n'})\Big|\\
    \leq& \Big|\max_{t,j} \{\theta_{t,j}^{(m)} \}\Big|\cdot \Big|\sum_{i=2}^{s}\sum_{j=1}^d X_{t,(i,j)}\Big|+\mathcal{O}(\sigma_0^2)\\
\end{align*}
where $\theta_{t,j}^{(m)}$ is the $j$-th element of vector $\bm{\theta}_t^{(m)}$ and $X_{t,(i,j)}$ is the $(i,j)$-th element of matrix $\mathbf{X}_{t}$. 

Then we apply Hoeffding's inequality to obtain
\begin{align*}
    \mathbb{P}\Big(|\sum_{i=2}^{s}\sum_{j=1}^d X_{t,(i,j)}|<s\cdot d\cdot \sigma_0\Big)\geq 1-2\exp{(-\frac{\sigma_0^2 s^2 d^2}{\|\mathbf{X}_{t,i}\|_{\infty}})}.
\end{align*}
As $X_{t,(i,j)}\sim\mathcal{N}(0,\sigma_t^2)$, we have $\|\mathbf{X}_{t,i}\|_{\infty}=\mathcal{O}(\sigma_t)$, indicating $\exp{(-\frac{\sigma_0^2 s^2 d^2}{\|\mathbf{X}_{t,i}\|_{\infty}})}=o(1)$. Therefore, with probability at least $1-o(1)$, we have $\big|\sum_{i=2}^{s}\sum_{j=1}^d X_{t,(i,j)}\big|=\mathcal{O}(\sigma_0)$. 
Consequently, we obtain
$\big|h_m(\mathbf{X}_t,\bm{\theta}_t^{(m)})-h_m(\tilde{\mathbf{X}}_t,\bm{\theta}_t^{(m)})\big|=\mathcal{O}(\sigma_0)$.

\subsection{Proof of Proposition~\ref{lemma:M_th}}\label{proof_lemma:M_th}
We first prove that if $|\mathcal{M}_n|\geq M_{th}$, there always exists an idle expert $m\in \mathcal{M}_{n_t}$ with $\gamma_t^{(m)}=1$ for any task $t$. Then we prove that $M=\Omega(NM_{th}\ln(\frac{1}{\delta}))$ makes $|\mathcal{M}_n|\geq M_{th}$ with probability at least $1-o(1)$.

We first prove that $|\mathcal{M}_n|\geq M_{th}$ guarantees at least an idle expert for any task arrival with $\bm{w}_t\in\mathcal{W}_n$. For ease of exposition, we named tasks as task $n$ for any task with $\bm{w}_t\in\mathcal{W}_n$ in the following.
In this case, we need to guarantee that the probability, of which task $n$ arrives more than $M_{th}$ times within $d_u$ time slots, is smaller than infinitesimal $\delta$. In other words, $M_{th}$ should ensure
\begin{align*}
    \text{Pr(Task $n$ arrives more than $M_{th}$ times in $d_u$ time slots)}< \delta.
\end{align*}
Let $A_n$ denote the number of times task $n$ arrives in $d_u$ time slots. Since each task is selected independently with probability $\frac{1}{B}$, we have $A_n\sim \text{Bionomial}(d,1/N)$, and we seek $\text{Pr}(A_n>M_{th})< \delta$.

Since $d_u$ is large, we can approximate $A_n$ by a normal distribution:
\begin{align*}
    A_n\approx \mathcal{N}(\frac{1}{N},\frac{1}{N}(1-\frac{1}{N})).
\end{align*}
Using the normal tail bound, we have
\begin{align*}
    \text{Pr}(A_n>M_{th})\approx \text{Pr}\left(\frac{A_n-d_u/N}{\sqrt{d/N(1-1/N)}}>\frac{A_n-M_{th}/N}{\sqrt{d/N(1-1/N)}}\right).
\end{align*}
Using the standard normal quantile $\Phi^{-1}(1-\delta)$, we obtain
\begin{align*}
    M_{th}\geq \frac{d_u}{N}+\Phi^{-1}(1-\delta)\sqrt{\frac{d}{N}(1-\frac{1}{N})},
\end{align*}
which guarantees at least an idle expert for task $n$ arrival during any $d_u$ time slots with probability $1-\delta$.

Finally, we prove that $M=\Omega(NM_{th}\ln(\frac{1}{\delta}))$ makes $|\mathcal{M}_n|\geq M_{th}$ with probability at least $1-o(1)$.

Given $M$ experts in total, the expected number of experts of each expert set $\mathcal{M}_n$ is $\frac{M}{N}$. Let $\alpha=1-\frac{M_{th}N}{M}$.
Then, by Chernoff-Hoeffding inequality, we obtain
\begin{align*}
    \mathbb{P}(|\mathcal{M}_n|< M_{th})&=\mathbb{P}\Big(|\mathcal{M}_n|<(1-\alpha)\frac{M}{N}\Big)\\&\leq \exp(-\alpha^2\frac{M}{2N})\\&=\exp(-\frac{(1-\frac{M_{th}N}{M})^2M}{2N}).
\end{align*}
To achieve $|\mathcal{M}_n|\geq M_{th}$ with probability at least $1-o(1)$, we let
\begin{align*}
    \exp(-\frac{(1-\frac{M_{th}N}{M})^2M}{2N})<\delta,
\end{align*}
solving which we obtain $M=\Omega(NM_{th}\ln(\frac{1}{\delta}))$.

\subsection{Proof of Proposition~\ref{prop:exploration}}\label{proof_prop:exploration}
To prove \Cref{prop:exploration}, we first propose the following lemmas. Then we prove \Cref{prop:exploration} based on these lemmas.

\begin{lemma}\label{lemma:gradient}
At any time $t\in\{1,\cdots, T_1\}$ with update, the gating network parameter of expert $m\in\mathbb{M}$ satisfies
\begin{align*}
    \langle\bm{\theta}_{t}^{(m)}-\bm{\theta}_{t-1}^{(m)}, \bm{v}_t\rangle=\begin{cases}
        -\mathcal{O}(\sigma_0), &\text{if expert $m$ completes}\\ &\text{a task with $\bm{w}_n$ at }t,\\
        \mathcal{O}(M^{-1}\sigma_0),&\text{otherwise.}
    \end{cases}
\end{align*}
\end{lemma}
\begin{proof}
According to the definition of locality loss in \cref{locality_loss}, we calculate
\begin{align}
    \nabla_{\bm{\theta}^{(m)}_t} \mathcal{L}_t=\frac{\partial \pi_{m_t}(\mathbf{X}_t,\mathbf{\Theta}_t)}{\partial\bm{\theta}^{(m)}_t}\|\bm{w}_t^{(m_t)}-\bm{w}_{t-1}^{(m_t)}\|_2.\label{partial_L_loc}
\end{align}
If $m=m_t$, we obtain
\begin{align}
    &\frac{\partial \pi_{m_t}(\mathbf{X}_t,\mathbf{\Theta}_t)}{\partial\bm{\theta}^{(m)}_t}\notag\\=&\pi_{m_t}(\mathbf{X}_t,\mathbf{\Theta}_t)\cdot \Big(\sum_{m'\neq m_t} \pi_{m'}(\mathbf{X}_t,\mathbf{\Theta}_t)\Big)\cdot \frac{\partial h_m(\mathbf{X}_t,\bm{\theta}_t^{(m)})}{\partial \bm{\theta}^{(m)}_t}\notag\\
    =&\pi_{m_t}(\mathbf{X}_t,\mathbf{\Theta}_t)\cdot \Big(\sum_{m'\neq m_t} \pi_{m'}(\mathbf{X}_t,\mathbf{\Theta}_t)\Big)\cdot \sum_{i\in[s_t]}\mathbf{X}_{t,i}.\label{partial_pi_mt}
\end{align}
If $m\neq m_t$, we obtain
\begin{align}
    \frac{\partial \pi_{m_t}(\mathbf{X}_t,\mathbf{\Theta}_t)}{\partial \bm{\theta}^{(m)}_t}&=-\pi_{m_t}(\mathbf{X}_t,\mathbf{\Theta}_t)\cdot \pi_{m}(\mathbf{X}_t,\mathbf{\Theta}_t)\cdot \sum_{i\in[s_t]}\mathbf{X}_{t,i}.\label{partial_pi_m'}
\end{align}
Based on \cref{partial_L_loc}, \cref{partial_pi_mt} and \cref{partial_pi_m'}, we obtain
\begin{align*}
    \sum_{m=1}^M \nabla_{\bm{\theta}_t^{(m)}} \mathcal{L}_t=\|\bm{w}_t^{(m_t)}-\bm{w}_{t-1}^{(m_t)}\|_2\sum_{m=1}^M \frac{\partial \pi_{m_t}(\mathbf{X}_t,\mathbf{\Theta}_t)}{\partial\bm{\theta}^{(m)}_t}=\mathbf{0}.
\end{align*}

Consequently, if $m=m_t$, we obtain
\begin{align*}
    &\langle\bm{\theta}_{t}^{(m_t)}-\bm{\theta}_{t-1}^{(m_t)}, \bm{v}_t\rangle\\=&-\eta \nabla_{\bm{\theta}^{(m_t)}_t} \mathcal{L}_t\cdot \bm{v}_t\\
    =&-O(\eta\cdot \|\bm{w}_t^{(m_t)}-\bm{w}_{t-1}^{(m_t)}\|_2)\\ &\cdot O(\pi_{m_t}(\mathbf{X}_t,\mathbf{\Theta}_t)\cdot \Big(\sum_{m'\neq m_t} \pi_{m'}(\mathbf{X}_t,\mathbf{\Theta}_t)\Big) \sum_{i\in[s_t]}\mathbf{X}_{t,i})\\
    =&-O(\sigma_0),
\end{align*}
based on the fact that $\|\bm{w}_t^{(m_t)}-\bm{w}_{t-1}^{(m_t)}\|_2=\sigma_0$, $O(\pi_m)=O(1)$ and $O(\eta)=O(1)$.

Since $\sum_{m=1}^M \nabla_{\bm{\theta}_t^{(m)}} \mathcal{L}_t=\mathbf{0}$, for any $m\neq m_t$, we obtain $\langle\bm{\theta}_{t}^{(m)}-\bm{\theta}_{t-1}^{(m)}, \bm{v}_t\rangle=O(M^{-1}\sigma_0)$.
\end{proof}

\begin{lemma}\label{lemma:theta_bound}
For any training round $t\in\{1,\cdots,T_1\}$, the gating network parameter of any expert $m$ satisfies $\|\bm{\theta}^{(m)}_{t}\|_{\infty}=\mathcal{O}(\sigma_0^{0.5})$.
\end{lemma}
\begin{proof}
Based on \Cref{lemma:gradient}, for any $t\in\{1,\cdots,T_1\}$ the accumulated update of $\bm{\theta}^{(m)}_{t}$ throughout the exploration stage satisfies 
\begin{align*}
    \|\bm{\theta}_{t}^{(m)}\|_{\infty}\leq \eta\cdot T_1\cdot \|\nabla_{\bm{\theta}^{(m)}_t} \mathcal{L}_t\|_{\infty}=\mathcal{O}(\sigma_0^{0.5}),
\end{align*}
based on the fact that $\bm{\theta}_{0}^{(m)}=\mathbf{0}$ at the initial time.
\end{proof}

For any $m\neq m'$, define $\delta_{\mathbf{\Theta}}=|h_{m}(\mathbf{X}_t,\bm{\theta}_t^{(m)})-h_{m'}(\mathbf{X}_t,\bm{\theta}_t^{(m)})|$. Then we obtain the following lemma.
\begin{lemma}\label{lemma:pi-hat_pi}
At any round $t$, if $\delta_{\mathbf{\Theta}_t}=o(1)$, it satisfies $|\pi_{m}(\mathbf{X}_t,\mathbf{\Theta}_t)-\pi_{m'}(\mathbf{X}_t,\mathbf{\Theta}_t)|=\mathcal{O}(\delta_{\mathbf{\Theta}})$. Otherwise,  $|\pi_{m}(\mathbf{X}_t,\mathbf{\Theta}_t)-\pi_{m'}(\mathbf{X}_t,\mathbf{\Theta}_t)|=\Omega(\delta_{\mathbf{\Theta}})$.
\end{lemma}
\begin{proof}
At any round $t$, we calculate
\begin{align*}
    &|\pi_m(\mathbf{X}_t,\mathbf{\Theta}_t)-\pi_{m'}(\mathbf{X}_t,\mathbf{\Theta}_t)|\\=&\Big|\pi_{m'}(\mathbf{X}_t,\mathbf{\Theta}_t)\exp{(h_m(\mathbf{X}_t,\bm{\theta}_t^{(m)})-h_{m'}(\mathbf{X}_t,\bm{\theta}_t^{(m')}))}\\&-\pi_{m'}(\mathbf{X}_t,\mathbf{\Theta}_t)\Big|\\
    =&\pi_{m'}(\mathbf{X}_t,\mathbf{\Theta}_t)\Big|\exp{(h_m(\mathbf{X}_t,\bm{\theta}_t^{(m)})-h_{m'}(\mathbf{X}_t,\bm{\theta}_t^{(m')}))}-1\Big|,
\end{align*}
where the first equality is by solving \cref{softmax}. Then if $\delta_{\mathbf{\Theta}_t}$ is close to $0$, by applying Taylor series with sufficiently small $\delta_{\mathbf{\Theta}}$, we obtain 
\begin{align*}
    &|\pi_m(\mathbf{X}_t,\mathbf{\Theta}_t)-\pi_{m'}(\mathbf{X}_t,\mathbf{\Theta}_t)|\\
    \approx &\pi_{m'}(\mathbf{X}_t,\mathbf{\Theta}_t)|h_{m}(\mathbf{X}_t,\bm{\theta}_t^{(m)})-h_{m'}(\mathbf{X}_t,\bm{\theta}_t^{(m')})|\\
    =&\mathcal{O}(\delta_{\mathbf{\Theta}}),
\end{align*}
where the last equality is because of $\pi_m(\tilde{\mathbf{X}}_t,\mathbf{\Theta}_t)\leq 1$.

While if $\delta_{\mathbf{\Theta}_t}$ is not sufficiently small, we obtain
\begin{align*}
    &|\pi_m(\mathbf{X}_t,\mathbf{\Theta}_t)-\pi_{m'}(\mathbf{X}_t,\mathbf{\Theta}_t)|\\>&\pi_{m'}(\mathbf{X}_t,\mathbf{\Theta}_t)|h_{m}(\mathbf{X}_t,\bm{\theta}_t^{(m)})-h_{m'}(\mathbf{X}_t,\bm{\theta}_t^{(m')})|\\
    =&\Omega(\delta_{\mathbf{\Theta}}).
\end{align*}
This completes the proof.
\end{proof}

\begin{lemma}\label{lemma:pi_consistence}
For any $n\neq n'$, if expert $m\in\mathcal{M}_n$, the following property holds at round $T_1$:
\begin{align}
    \pi_m(\mathbf{X}_n,\mathbf{\Theta}_t)>\pi_m(\mathbf{X}_{n'},\mathbf{\Theta}_t), \forall m\in\mathbb{M}.\label{pi_consistence}
\end{align}
\end{lemma}
\begin{proof}
We first calculate
\begin{align*}
    &|h_m(\mathbf{X}_n,\bm{\theta}_t^{(m)})-h_m(\mathbf{X}_{n'},\bm{\theta}_t^{(m)})|\\=&
    |\langle \bm{\theta}_t^{(m)}, \bm{v}_n-\bm{v}_{n'}\rangle|\\
    =&\|\bm{\theta}_t^{(m)}\|_{\infty} \cdot \|\bm{v}_n-\bm{v}_{n'}\|_{\infty}\\
    =&\mathcal{O}(\sigma_0^{0.5}),
\end{align*}
where the last equality is because of $\|\bm{\theta}_t^{(m)}\|_{\infty}=\mathcal{O}(\sigma_0^{0.5})$ in \Cref{lemma:theta_bound} and $\|\bm{v}_n-\bm{v}_{n'}\|_{\infty}=\mathcal{O}(1)$. Then according to \Cref{lemma:pi-hat_pi}, we obtain
\begin{align}
    |\pi_m(\mathbf{X}_n,\mathbf{\Theta}_t)-\pi_m(\mathbf{X}_{n'},\mathbf{\Theta}_t)|=\Omega(\sigma_0^{0.5})\label{true_max_pi}
\end{align}
for any $t\in\{1,\cdots, T_1\}$, which completes the proof of \Cref{lemma:pi_consistence}.
\end{proof}

\begin{lemma}\label{lemma:fairness}
At the end of the exploration stage, with probability at least $1-\delta$, the fraction of tasks dispatched to any expert $m\in[M]$ satisfies
\begin{align}
    \Big|f_{T_1}^{(m)}-\frac{1}{M}\Big|=\mathcal{O}(\eta^{0.5}M^{-1}).\label{fT_1=1/M}
\end{align}
\end{lemma}
\begin{proof}
By the symmetric property, we have that for any $m\in[M]$, $\mathbb{E}[f_{T_1}^{(m)}]=\frac{1}{M}$.

By Hoeffding's inequality, we obtain 
\begin{align*}
    \mathbb{P}(|f_{T_1}^{(m)}-\frac{1}{M}|\leq \epsilon)\geq 1-2\exp{(-2\epsilon^2T_1)}.
\end{align*}
Then we further obtain
\begin{align*}
    &\mathbb{P}(|f_{T_1}^{(m)}-\frac{1}{M}|\leq \epsilon,\forall m\in[M])\\\geq& (1-2\exp{(-2\epsilon^2T_1)})^M\\
    \geq& 1-2M\exp{(-2\epsilon^2T_1)}).
\end{align*}

Let $\delta=1-2M\exp{(-2\epsilon^2T_1)})$. Then we obtain $\epsilon=\mathcal{O}(\eta^{0.5} M^{-1})$. Subsequently, there is a probability of at least $1-\delta$ that $\big|f_{T_1}^{(m)}-\frac{1}{M}\big|=\mathcal{O}(\eta^{0.5}M^{-1})$.
\end{proof}

\begin{lemma}\label{lemma:exploration_phase}
At the end of the exploration stage, i.e., $t=T_1$, the following property holds
\begin{align*}
    \Big\|\bm{\theta}_{T_1}^{(m)}-\bm{\theta}_{T_1}^{(m')}\Big\|_{\infty}=\mathcal{O}(\eta^{-0.5}\sigma_0),
\end{align*}
for any $m, m'\in [M]$ and $m\neq m'$.
\end{lemma}
\begin{proof}
Based on \Cref{lemma:gradient} and \Cref{lemma:fairness} and their corresponding proofs above, we can prove \Cref{lemma:exploration_phase} below.

Define $f_t^{(m)}:=\frac{1}{t}\sum_{\tau=1}^t\mathds{1}\{m_{\tau}=m\}$ as the fraction of tasks dispatched to expert $m$ since $t=1$.
For experts $m$ and $m'$, they are selected by the router for $T_1\cdot f_{T_1}^{(m)}$ and $T_1\cdot f_{T_1}^{(m')}$ times during the exploration stage, respectively. Therefore, we obtain
\begin{align*}
    \|\bm{\theta}_{T_1}^{(m)}\|_{\infty}&=f_{T_1}^{(m)}\cdot T_1\cdot \mathcal{O}(\sigma_0)-(1-f_{T_1}^{(m)})\cdot T_1\cdot \mathcal{O}(M^{-1}\sigma_0),\\
    \|\bm{\theta}_{T_1}^{(m')}\|_{\infty}&=f_{T_1}^{(m')}\cdot T_1\cdot \mathcal{O}(\sigma_0)-(1-f_{T_1}^{(m')})\cdot T_1\cdot \mathcal{O}(M^{-1}\sigma_0).
\end{align*}

Then by \cref{update_theta} and \Cref{lemma:gradient}, we calculate 
\begin{align*}
    &\Big\|\bm{\theta}_{T_1}^{(m)}-\bm{\theta}_{T_1}^{(m')}\big\|_{\infty}\\=& \Big| \big(f_{T_1}^{(m)}-f_{T_1}^{(m')}\big)\cdot T_1\cdot  \mathcal{O}(\sigma_0)-\big((1-f_{T_1}^{(m)})\\ &-(1-f_{T_1}^{(m')})\big)\cdot T_1\cdot \mathcal{O}(M^{-1}\sigma_0)\big|\\
    =& |f_{T_1}^{(m)}-f_{T_1}^{(m')}|\cdot T_1\cdot  \mathcal{O}(\sigma_0)\\
    =&\mathcal{O}(\eta^{-0.5} \sigma_0), 
\end{align*}
where the first equality is derived based on the update steps in \Cref{lemma:gradient}, and the last equality is because of $T_1\cdot |f_{T_1}^{(m)}-f_{T_1}^{(m')}|=\mathcal{O}(\eta^{-0.5})$ by \cref{fT_1=1/M} in \Cref{lemma:fairness}.
\end{proof}

According to \Cref{lemma:gradient}, if an expert $m$ completes a task at $t<T_1$, its gating output $(\bm{\theta}_t^{(m)})^\top \bm{v}_n$ with respect to feature vector $\bm{v}_n$ will be reduced compared to $(\bm{\theta}_{t-1}^{(m)})^\top \bm{v}_n$. 
Meanwhile, the other experts' outputs are increased. Then for the next task arrival of ground truth $\bm{w}_n$, it will be routed to another expert. 

Let $Y_t^{(m)}$ denote the number of times that expert $m$ is routed until time $t$.
By Cheroff-Hoeffding inequality, we obtain
\begin{align*}
    \mathbb{P}(Y^{(m)}_t\geq 1,\forall m\in\mathbb{M})&\geq \Big(1-(1-\frac{1}{M})^t\Big)^M\\
    &\geq 1-M(1-\frac{1}{M})^t\\
    &\geq 1-M\exp(-\frac{t}{M})\\
    & \geq 1-\delta,
\end{align*}
solving which we obtain $t\geq \lceil M\ln(\frac{M}{\delta})\rceil$. Therefore, given $T_1=d_u+\lceil \eta^{-1}\sigma_0^{-0.5}M\ln(\frac{M}{\delta})\rceil$, each expert has been explored at least $\eta^{-1}\sigma_0^{-0.5}$ times during the expert-exploration stage (i.e., $t<T_1$).
Here the first inequality is because the probability of $Y_t^{(m)}\geq 1$ under routing strategy (\ref{m_t}) is greater than that under randomly routing strategy with identical probability $\frac{1}{M}$. The second inequality is because $(1-\frac{1}{M})^t\rightarrow 0$ and the third inequality is derived by Binomial theorem. 

Then based on \Cref{lemma:pi_consistence}, we can prove that each expert $m$ stabilizes within an expert set $\mathcal{M}_n$ by equivalently proving the following properties:
\begin{align}
    \pi_m(\mathbf{X}_n,\mathbf{\Theta}_t)>\pi_{m'}(\mathbf{X}_{n},\mathbf{\Theta}_t),\ \ \pi_{m'}(\mathbf{X}_{n'},\mathbf{\Theta}_t)>\pi_m(\mathbf{X}_{n'},\mathbf{\Theta}_t),\label{max_pi_exploration}
\end{align}
where $\mathbf{X}_n$ and $\mathbf{X}_{n'}$ contain feature signals $\bm{v}_n$ and $\bm{v}_{n'}$, respectively. 

Next, we prove \cref{max_pi_exploration} by contradiction. Assume there exist two experts $m\in\mathcal{M}_n$ and $m'\in\mathcal{M}_{n'}$ such that 
\begin{align*}
    \pi_m(\mathbf{X}_{n},\mathbf{\Theta}_t)>\pi_{m'}(\mathbf{X}_{n},\mathbf{\Theta}_t),\ \ \pi_m(\mathbf{X}_{n'},\mathbf{\Theta}_t)>\pi_{m'}(\mathbf{X}_{n'},\mathbf{\Theta}_t),
\end{align*}
which is equivalent to
\begin{align}
    \pi_m(\mathbf{X}_{n},\mathbf{\Theta}_t)>\pi_m(\mathbf{X}_{n'},\mathbf{\Theta}_t)>\pi_{m'}(\mathbf{X}_{n'},\mathbf{\Theta}_t)>\pi_{m'}(\mathbf{X}_{n},\mathbf{\Theta}_t) ,\label{wrong_max_pi}
\end{align}
because of $\pi_{m
}(\mathbf{X}_{n},\mathbf{\Theta}_t)>\pi_{m}(\mathbf{X}_{n'},\mathbf{\Theta}_t)$ and $\pi_{m
'}(\mathbf{X}_{n'},\mathbf{\Theta}_t)>\pi_{m'}(\mathbf{X}_{n},\mathbf{\Theta}_t)$ based on the definition of expert set $\mathcal{M}_n$ in \cref{M_n}. Then we prove \cref{wrong_max_pi} does not exist at $t=T_1$.

For task $t=T_1$, we calculate
\begin{align}
    &|h_m(\mathbf{X}_{n},\bm{\theta}_{T_1}^{(m)})-h_{m'}(\mathbf{X}_{n},\bm{\theta}_{T_1}^{(m)})|\notag\\ \leq& \|\bm{\theta}_{T_1}^{(m)}-\bm{\theta}_{T_1}^{(m')}\|_{\infty} \|\bm{v}_{n}\|_{\infty}\notag\\
    =& \mathcal{O}(\sigma_0 \eta^{-0.5}),\label{wrong_pi_2}
\end{align}
where the first inequality is derived by union bound, and the second equality is because of $\|\bm{v}_{n}\|_{\infty}=\mathcal{O}(1)$ and $\|\bm{\theta}_{T_1}^{(m)}-\bm{\theta}_{T_1}^{(m')}\|_{\infty}=\mathcal{O}(\sigma_0 \eta^{-0.5})$ derived in \Cref{lemma:exploration_phase} at $T_1$.

Then according to \Cref{lemma:pi-hat_pi} and \cref{wrong_pi_2}, we obtain 
\begin{align}
    |\pi_m(\mathbf{X}_{n},\mathbf{\Theta}_{T_1})-\pi_{m'}(\mathbf{X}_{n},\mathbf{\Theta}_{T_1})|=\mathcal{O}(\sigma_0 \eta^{-0.5}).\label{wrong_pi_3}
\end{align}

Based on \cref{wrong_max_pi}, we further calculate
\begin{align*}
    &|\pi_m(\mathbf{X}_n,\mathbf{\Theta}_{T_1})-\pi_{m'}(\mathbf{X}_{n},\mathbf{\Theta}_{T_1})|\\ \geq &|\pi_m(\mathbf{X}_n,\mathbf{\Theta}_{T_1})-\pi_{m}(\mathbf{X}_{n'},\mathbf{\Theta}_{T_1})|\\
    =&\Omega(\sigma_0^{0.5}),
\end{align*}
where the first inequality is derived by \cref{wrong_max_pi}, and the last equality is derived in \cref{true_max_pi}. This contradicts with \cref{wrong_pi_3} as $\sigma_0 \eta^{-0.5}<\sigma_0^{0.5}$ given $\eta=\mathcal{O}(\sigma_0^{0.5})$. Therefore, \cref{wrong_max_pi} does not exist for $t=T_1$, and \cref{max_pi_exploration} is true for $t=T_1$.
Consequently, at time $T_1$, the router can assign tasks within the same cluster $\mathcal{W}_n$ to any expert $m\in\mathcal{M}_n$, and (\ref{hm-hm'}) is obviously true based on (\ref{true_max_pi}).

\subsection{Proof of Proposition~\ref{prop:expert_learning}}\label{proof_prop:expert_learning}
For $t>T_1$, any task $n_t$ with $\bm{w}_{n_t}\in\mathcal{W}_k$ will be routed to the correct expert $m\in\mathcal{M}_k$. Let $\bm{w}^{(m)}$ denote the minimum $\ell^2$-norm offline solution for expert $m$. Based on the update rule of $\bm{w}_t^{(m_t)}$ in \cref{update_wt}, we calculate
\begin{align*}
    &\bm{w}^{(m_t)}_t-\bm{w}^{(m_t)}\\=&\bm{w}_{t-1}^{(m_t)}+\mathbf{X}_t(\mathbf{X}_t^\top \mathbf{X}_t)^{-1}(\mathbf{y}_t-\mathbf{X}_t^{\top}\bm{w}_{t-1}^{(m_t)})-\bm{w}^{(m_t)}\\   
    =&(\bm{I}-\mathbf{X}_t(\mathbf{X}_t^\top \mathbf{X}_t)^{-1}\mathbf{X}_t^{\top})\bm{w}_{t-1}^{(m_t)}\\&+\mathbf{X}_t(\mathbf{X}_t^\top \mathbf{X}_t)^{-1}\mathbf{X}^{\top}_t\bm{w}^{(m_t)}-\bm{w}^{(m_t)}\\
    =&(\bm{I}-\mathbf{X}_t(\mathbf{X}_t^\top \mathbf{X}_t)^{-1}\mathbf{X}_t^{\top})(\bm{w}_{t-1}^{(m_t)}-\bm{w}^{(m_t)}),
\end{align*}
where the second equality is because of $\mathbf{y}_t=\mathbf{X}^{\top}_t\bm{w}^{(m_t)}$. Define $\bm{P}_t=\mathbf{X}_t(\mathbf{X}_t^\top \mathbf{X}_t)^{-1}\mathbf{X}_t^{\top}$ for task $n_t$, which is the projection operator on the solution space $\bm{w}_{n_t}$. Then we obtain
\begin{align*}
    \bm{w}^{(m)}_t-\bm{w}^{(m)}=(\bm{I}-\bm{P}_t)\cdots (\bm{I}-\bm{P}_{T_1+1})(\bm{w}_{T_1}^{(m)}-\bm{w}^{(m)})
\end{align*}
for each expert $m\in[M]$.

Since orthogonal projections $\bm{P}_t$'s are non-expansive operators, $\forall t\in\{T_1+1,\cdots ,T\},$, it also follows that
\begin{align*}
 \|\bm{w}^{(m)}_t-\bm{w}^{(m)}\|\leq\|\bm{w}^{(m)}_{t-1}-\bm{w}^{(m)}\| \leq\cdots\leq \|\bm{w}^{(m)}_{T_1}-\bm{w}^{(m)}\|.
\end{align*}
As the solution spaces $\mathcal{W}_k$ is fixed for each expert $m\in\mathcal{M}_k$, we further obtain
\begin{align*}
    \|\bm{w}_t^{(m)}-\bm{w}_{T_1+1}^{(m)}\|_{\infty}&=\|\bm{w}_t^{(m)}-\bm{w}^{(m)}+\bm{w}^{(m)}-\bm{w}_{T_1+1}^{(m)}\|_{\infty}\\
    &\leq \|\bm{w}_t^{(m)}-\bm{w}^{(m)}\|_{\infty}+\|\bm{w}_{T_1+1}^{(m)}-\bm{w}^{(m)}\|_{\infty}\\
    &\leq \max_{\bm{w}_n,\bm{w}_{n'}\in \mathcal{W}_k}\|\bm{w}_n-\bm{w}_{n'}\|_{\infty}\\ &=\mathcal{O}(\sigma_0^{2}),
\end{align*}
where the first inequality is derived by the union bound, the second inequality is because of the orthogonal projections for the update of $\bm{w}_t^{(m)}$ per task, and the last equality is because of $\|\bm{w}_n-\bm{w}_{n'}\|_{\infty}=\mathcal{O}(\sigma_0^{2})$ for any two ground truths in the same set $\mathcal{W}_k$.

\subsection{Proof of Proposition~\ref{prop:GT_benchmark}}\label{proof_prop:GT_benchmart}
If the MEC network operator always chooses the nearest or the most powerful expert for each task arrival, the task $n$ routed to each expert is random. Then we use the same result of (\ref{E[wt-wi]}) in \Cref{lemma:error_cases} to prove \Cref{prop:GT_benchmark}.

We calculate the generalization error as:
\begin{align*}
    &\mathbb{E}[G_T]\\=&\frac{1}{T}\sum_{i=1}^T\mathbb{E}[\|\bm{w}_T^{(m_i)}-\bm{w}_{n_i}\|^2]\\
    =&\frac{1}{T}\sum_{i=1}^T\Big(r^{L_T^{(m_t)}}\|\bm{w}_{i}\|^2+\sum_{l=1}^{T}(1-r)r^{L_{T-l}^{(m_l)}}\mathbb{E}[\|\bm{w}_{n_l}-\bm{w}_{n_i}\|^2]\Big)\\
    =&\frac{1}{T}\sum_{t=1}^T r^{L_{T}^{(m_t)}}\|\bm{w}_t\|^2\\&+\frac{1}{T}\sum_{t=1}^T(1-r^{L_T^{(m_t)}})\mathbb{E}\Big[\|\bm{w}_{n}-\bm{w}_{n'}\|^2\Big|n,n'\in[N]\Big],
\end{align*}
where the second equality is because of \cref{E[wt-wi]} in \Cref{lemma:error_cases}, and the last equality is because any tasks $\bm{w}_{n_l}$ and $\bm{w}_{n_i}$ are randomly sampled from set $[N]$.

\subsection{Proof of Theorem~\ref{thm:error_algorithm}}\label{proof_thm:error_algorithm}
For expert $m$, let $\tau^{(m)}(l)\in\{1,\cdots,T_1\}$ represent the training round of the $l$-th time that the router selects expert $m$ during the exploration stage. For instance, $\tau^{(1)}(2)=5$ indicates that round $t=5$ is the second time the router selects expert 1. 
\begin{lemma}\label{lemma:error_cases}
At any round $t\in\{T_1+1,\cdots,T\}$, for $i\in\{T_1+1,\cdots, t\}$, we have
\begin{align*}
    \|\bm{w}_t^{(m_i)}-\bm{w}_{n_i}\|^2=\|\bm{w}_{T_1}^{(m_i)}-\bm{w}_{n_i}\|^2+O(\sigma_0^2).
\end{align*}
While at any round $t\in\{1,\cdots,T_1\}$, for any $i\in\{1,\cdots, t\}$, we have
\begin{align}
    &\mathbb{E}[\|\bm{w}_t^{(m_i)}-\bm{w}_{n_i}\|^2]\label{E[wt-wi]}\\ = &r^{L_t^{(m_i)}}\mathbb{E}[\|\bm{w}_{n_i}\|^2]+\sum_{l=1}^{L_t^{(m_i)}}(1-r)r^{L_t^{(m_i)}-l}\mathbb{E}[\|\bm{w}_{\tau^{(m_i)}(l)}-\bm{w}_{n_i}\|^2],\notag
\end{align}
where $L^{(m_i)}_t=t\cdot f_t^{(m_i)}$ and $r=1-\frac{s}{d}$.
\end{lemma}
\begin{proof}
Based on \Cref{prop:exploration} and \Cref{prop:expert_learning}, we can easily obtain $\|\bm{w}_t^{(m_i)}-\bm{w}_{n_i}\|^2=\|\bm{w}_{T_1}^{(m_i)}-\bm{w}_{n_i}\|^2+O(\sigma_0^2)$, such that we skip its proof here.

For any round $t\in\{1,\cdots,T_1\}$, define $\mathbf{P}_t=\mathbf{X}_t(\mathbf{X}_t^{\top}\mathbf{X}_t)^{-1}\mathbf{X}_t^{\top}$ for task $t$. At current time $t$, there are totally $L_t^{(m)}=t\cdot f_t^{(m)}$ tasks routed to expert $m$, where $f_t^{(m)}=\frac{1}{t}\sum_{\tau=1}^t\mathds{1}\{m_{\tau}=m\}$ is defined in \Cref{lemma:exploration_phase}.

Based on the update rule of $\bm{w}_t^{(m)}$ in \cref{update_wt}, we calculate 
\begin{align*}
    &\|\bm{w}_t^{(m_i)}-\bm{w}_{n_i}\|^2\\=&\|\bm{w}_{\tau^{(m_i)}(L_t^{(m_i)})}^{(m_i)}-\bm{w}_{n_i}\|^2\\=&\|(\mathbf{I}-\mathbf{P}_{t})\bm{w}_{\tau^{(m_i)}(L_t^{(m_i)}-1)}^{(m_i)}+\mathbf{P}_{t}\bm{w}_{\tau^{m_i}(L_t^{(m_i)})}-\bm{w}_{n_i}\|^2\\=&\|(\mathbf{I}-\mathbf{P}_{t})(\bm{w}_{\tau^{(m_i)}(L_t^{(m_i)}-1)}^{(m_i)}-\bm{w}_{n_i})\\&+\mathbf{P}_{t}(\bm{w}_{\tau^{m_i}(L_t^{(m_i)})}-\bm{w}_{n_i})\|^2,
\end{align*}
where the first equality is because there is no update of $\bm{w}_t^{(m_i)}$ for $t\in\{\tau^{(m_i)}(L_t^{(m_i)}),\cdots, t\}$, and the second equality is by \cref{update_wt}. 

As $\mathbf{P}_t$ is the orthogonal projection matrix for the row space of $\mathbf{X}_t$, based on the rotational symmetry of the standard normal distribution, it follows that $\mathbb{E}\Big[\|\mathbf{P}_t(\bm{w}_{\tau^{m_i}(L_t^{(m_i)})}-\bm{w}_{n_i})\|\Big]=\frac{s}{d}\|\bm{w}_{\tau^{m_i}(L_t^{(m_i)})}-\bm{w}_{n_i}\|^2$. Then we further calculate
\begin{align*}
    &\mathbb{E}[\|\bm{w}_t^{(m_i)}-\bm{w}_{n_i}\|^2]\\=&(1-\frac{s}{d})\mathbb{E}[\|\bm{w}_{\tau^{(m_i)}(L_t^{(m_i)}-1)}^{(m_i)}-\bm{w}_{n_i}\|^2]\\&+\frac{s}{d}\mathbb{E}[\|\bm{w}_{\tau^{(m_i)}(L_t^{(m_i)})}-\bm{w}_{n_i}\|^2]\\
    =&(1-\frac{s}{d})^{L_t^{(m_i)}}\mathbb{E}[\|\bm{w}_0^{(m_i)}-\bm{w}_{n_i}\|^2]\\&+\sum_{l=1}^{L_t^{(m_i)}}(1-\frac{s}{d})^{L_t^{(m_i)}-l}\frac{s}{d}\mathbb{E}[\|\bm{w}_{\tau^{(m_i)}(L_t^{(m_i)})}-\bm{w}_{n_i}\|^2]\\
    =&r^{L_t^{(m_i)}}\mathbb{E}[\|\bm{w}_{n_i}\|^2]+\\&\sum_{l=1}^{L_t^{(m_i)}}(1-r)r^{L_t^{(m_i)}-l}\mathbb{E}[\|\bm{w}_{\tau^{(m_i)}(l)}-\bm{w}_{n_i}\|^2],
\end{align*}
where the second equality is derived by iterative calculation, and the last equality is because of $\bm{w}_0^{(m)}=\mathbf{0}$ for any expert $m$. Here we denote by $r=1-\frac{s}{d}$ to simplify notations.
\end{proof}

Based on \Cref{lemma:error_cases}, we calculate
\begin{align*}
    &\mathbb{E}[G_T]\\=&\frac{1}{T}\sum_{t=1}^T\mathbb{E}[\|\bm{w}_T^{(m_t)}-\bm{w}_{n_t}\|^2]\\
    =&\frac{1}{T}\sum_{t=1}^{T}\mathbb{E}[\|\bm{w}_{T_1}^{(m_t)}-\bm{w}_{n_t}\|^2]\\
    \leq &\frac{1}{T}\sum_{t=1}^T\Big(r^{L_{T_1}^{(m_t)}}\|\bm{w}_{t}\|^2\\&+\sum_{l=1}^{L_{T_1}^{(m_t)}}(1-r)r^{L_{T_1}^{(m_t)}-l}\mathbb{E}[\|\bm{w}_{\tau^{(m_t)}(l)}-\bm{w}_{n_t}\|^2]+O(\sigma_0^2)\Big)\\
    =&\frac{1}{T}\sum_{t=1}^{T}r^{L_{T}^{(m_t)}}\|\bm{w}_{t}\|^2+\frac{1}{T}\sum_{t=1}^{T}(1-r^{L_{T_1}^{(m_t)}})\\&\cdot r^{L_T^{(m_t)}-L_{T_1}^{(m_t)}}\mathbb{E}\Big[\|\bm{w}_{n}-\bm{w}_{n'}\|^2\Big|n,n'\in[N]\Big]+\mathcal{O}(\sigma_0^{2}),
\end{align*}
where the inequality is because of $\mathbb{E}[\|\bm{w}_{T_1}^{(m_t)}-\bm{w}_{n_t}\|^2]=O(\sigma_0^2)$ for $t\geq T_1+1$ derived in \Cref{lemma:error_cases}. This completes the proof of \Cref{thm:error_algorithm}.
\balance

\end{document}